\documentclass[USenglish]{article}

\usepackage[preprint]{jmlrish}

\usepackage[utf8]{inputenc} %
\usepackage[T1]{fontenc}    %
\usepackage{babel}
\usepackage{csquotes}
\usepackage{url}            %
\usepackage{booktabs}       %
\usepackage{amsmath,amsfonts,amsthm,amssymb}
\usepackage{nicefrac}       %
\usepackage{microtype}      %
\usepackage{xcolor}         %
\usepackage{mathtools}
\usepackage{thmtools,thm-restate}
\usepackage{bbm}  %
\usepackage{enumitem}
\usepackage{thm-restate}
\usepackage{subcaption}

\usepackage[colorinlistoftodos]{todonotes}

\definecolor{citecolor}{RGB}{0,180,0}
\definecolor{linkcolor}{RGB}{180,0,0}
\definecolor{urlcolor}{RGB}{0,0,180}
\usepackage[colorlinks,citecolor=citecolor,linkcolor=linkcolor,urlcolor=urlcolor]{hyperref}

\usepackage[capitalize,noabbrev]{cleveref}

\usepackage[style=authoryear,sortcites,sorting=ynt,maxcitenames=2,maxbibnames=10,maxalphanames=5,useprefix,uniquelist=minyear,dashed=false]{biblatex}
\renewbibmacro{in:}{} %
\DefineBibliographyStrings{english}{urlseen={visited in}}
\DeclareFieldFormat[unpublished,misc]{title}{\mkbibquote{#1}}  %
\let\citet\textcite
\let\Citet\Textcite
\let\citep\parencite

\newcommand{\httpurl}[1]{\href{http://#1}{\nolinkurl{#1}}}
\newcommand{\httpsurl}[1]{\href{https://#1}{\nolinkurl{#1}}}

\DeclareMathOperator{\rank}{rank}

\newcommand{\norm}[1]{\left\lVert #1 \right\rVert}
\newcommand{\Norm}[1]{\lVert #1 \rVert}

\newcommand{\N}{\mathbb{N}}
\newcommand{\Z}{\mathbb{Z}}

\newcommand{\OP}{\mathit{op}}

\newcommand{\cK}{\mathcal{K}}
\newcommand{\cH}{\mathcal{H}}
\newcommand{\cR}{\mathcal{R}}

\newcommand{\cD}{\mathcal{D}}

\newcommand{\bE}{\mathbb{E}}
\newcommand{\R}{\mathbb{R}}
\newcommand{\F}{\mathcal{F}}
\renewcommand{\L}{\mathcal{L}}
\newcommand{\I}{\mathcal{I}}

\let\pr\Pr

\DeclareMathOperator*{\Var}{\mathrm{Var}}
\DeclareMathOperator*{\E}{\bE}      %

\DeclareMathOperator{\rad}{rad}

\DeclareMathOperator{\Tr}{Tr}

\DeclareMathOperator*{\argmin}{arg\,min}

\newtheorem{theorem}{Theorem}
\newtheorem{corollary}{Corollary}
\newtheorem{lemma}{Lemma}
\newtheorem{proposition}{Proposition}

\theoremstyle{definition}
\newtheorem{defn}{Definition}
\newtheorem{remark}{Remark}
\newtheorem{example}{Example}

\newcommand{\wols}{\hat{w}_{\mathrm{OLS}}}

\title{Optimistic Rates: A Unifying Theory for Interpolation Learning\\and Regularization in Linear Regression}
\author{%
Lijia Zhou\footnotemark[1]\; \addr{Department of Statistics, University of Chicago} \email{zlj@uchicago.edu}
\AND
Frederic Koehler\thanks{These authors contributed equally.}\; \addr{Simons Institute, University of California at Berkeley} \email{fkoehler@berkeley.edu}
\AND
Danica J.\ Sutherland \addr{University of British Columbia; Alberta Machine Intelligence Institute} \email{dsuth@cs.ubc.ca}
\AND
Nathan Srebro \addr{Toyota Technological Institute at Chicago} \email{nati@ttic.edu}
}

\begin{document}

\maketitle
\setcounter{footnote}{0}
\begin{center}
\vspace{-12mm}
{{Collaboration on the Theoretical Foundations of Deep Learning} (\httpsurl{deepfoundations.ai})}
\vspace{3mm}
\end{center}

\begin{abstract}
We study a localized notion of uniform convergence known as an ``optimistic rate'' \citep{panchenkooptimistic,srebro2010optimistic} for linear regression with Gaussian data. Our refined analysis avoids the hidden constant and logarithmic factor in existing results, which are known to be crucial in high-dimensional settings, especially for understanding interpolation learning. As a special case, our analysis recovers the guarantee from \citet{uc-interpolators}, which tightly characterizes the population risk of low-norm interpolators under the benign overfitting conditions. Our optimistic rate bound, though, also analyzes predictors with arbitrary training error. This allows us to recover some classical statistical guarantees for ridge and LASSO regression under random designs, and helps us obtain a precise understanding of the excess risk of near-interpolators in the over-parameterized regime.
\end{abstract}

\section{Introduction}
One of the core mysteries behind the success of deep learning is that a neural network with a huge number of parameters can be trained with little to no regularization to fit noisy observations, and yet can still achieve good generalization on unseen data points. Even more mind-boggling is the observation that models with larger parameter counts actually tend to generalize \emph{better} \citep{ZBHRV:rethinking,NTS:real-inductive-bias,reconcile:interpolation}. This turns out to be a quite universal phenomenon, not unique to deep learning \citep{BHM:perfect, BHX:two-models, hastie2019surprises}. As over-parameterized models become more and more important in applications, it seem imperative to understand the mathematical reasons behind their success. 

In high-dimensional settings, there are usually possible solutions with low training error but very high population risk, and so any analysis based only on the number of parameters will be extremely loose. To explain over-parameterized learning, we need some alternative measure of complexity. Finding the relevant complexity measure of a neural network remains an open question, but we have come to understand that the appropriate complexity measure for linear regression is the norm of the coefficients. Much recent work \citep[e.g.][]{bartlett2020benign, tsigler2020benign, BHX:two-models, hastie2019surprises, JLL:basis-pursuit, junk-feats, muthukumar:interpolation, negrea:in-defense} has considered linear regression as a testbed problem which also exhibits some of the surprising behaviors found in deep learning. In particular, \citet{bartlett2020benign} show that it is possible for the minimal norm interpolator $\hat{w}$ to be consistent even when the number of dimensions grows much faster than the sample size. 

A very natural idea to recover this fact is the following: we can consider the set of predictors with norm smaller than $\|\hat{w}\|$, and argue that the difference between training error and population error is small uniformly for all predictors in this set. Because this set is simple in the sense that all predictors have small norm, we can hope for a uniform law of large numbers to show that the population risk of the minimal norm interpolator is also small. This idea, known as uniform convergence, has been the core workhorse of learning theory for decades. Unfortunately, there are lower bounds that show this approach cannot explain consistency in many natural high-dimensional problems \citep{NK:uniform, junk-feats, negrea:in-defense, BL:failures}. At a high level, this is because the norm required to perfectly fit the noisy labels need to scale with the sample size, and so the set of predictors with norm smaller than $\|\hat{w}\|$ can actually be quite large, and in particular will include predictors with high training error. To sidestep these negative results, \citet{junk-feats} argue that we should focus on upper bounds only for predictors with low training error. \Citet{uc-interpolators} subsequently show that if we only consider the low-norm predictors with \emph{exactly zero} training error, then a uniform convergence argument can actually tightly characterizes the population risk of low-norm interpolators in Gaussian linear regression. 

Though their works highlight the importance of localized uniform convergence and very clearly demonstrates that it is sufficient for interpolation learning, in practice we do not only care about \emph{exact} interpolators. For example, there can be interesting high dimensional settings where interpolation is not possible. When interpolation is possible, we can also obtain good non-interpolating predictors by early stopping or some amount of regularization. Even if we intend to perfectly memorize the labels, numerical precision issues will likely prevent us from fitting them to literally zero error. Thus, we want a more general notion of risk-dependent uniform convergence that is robust to non-interpolation. In the context of linear regression, we want to understand the population risk of any low-norm predictor with small, but not exactly zero training error. 

In this paper, we revisit the ``optimistic rate'' bound of \citet{srebro2010optimistic}, and perform a tighter analysis based on Gordon's comparison inequality for Gaussian processes \citep{gordon1985some,thrampoulidis2015regularized}.
Our new analysis is tight enough to recover the consistency result of the minimal-norm interpolator from \citet{bartlett2020benign} and \citet{uc-interpolators} for Gaussian linear regression, which previous work on optimistic rates cannot achieve due to hidden constants and logarithmic factors.\footnote{A more detailed discussion can be found at the beginning of \cref{sec:optimistic-rate}.} At the same time, our result allows us to have a very precise and accurate understanding of the finite-sample risk of non-interpolating estimators. For example, our upper bound for the ordinary least square estimator matches the exact expectation formula given by \citet{hastie2019surprises} in the proportional scaling limit, even though the estimator is not consistent. In \cref{sec:applications}, we also apply our generalization framework to analyze ridge and LASSO regression. We show that it is possible to understand classical statistical theory as well as recent progress in interpolation learning under the same unified framework of optimistic rates.

\section{Problem Setting}

\paragraph{Notation.}
We use $\Norm\cdot_p$ for the $\ell_p$ norm, $\norm{x}_p = \left(\sum_i |x_i|^p\right)^{1/p}$. For a positive semidefinite matrix $A$, the \emph{Mahalanobis (semi-)norm} is $\norm{x}_A^2 := \langle x, A x \rangle$.
For a matrix $A$ and set $S$,
$A S$ denotes the set $\{ A x : x \in S \}$.
We always use $\max_{x \in S} f(x)$ to be $-\infty$ when $S$ is empty, and similarly $\min_{x \in S} f(x)$ to be $\infty$. We use $a \vee b$ to denote the maximum between $a$ and $b$ and $a \wedge b$ to denote the minimum. We use standard $O(\cdot)$ notation, and $a \lesssim b$ for inequality up to an absolute constant.

\paragraph{Data model.} 
We assume that the data $(X,Y)$ is generated as
\begin{equation} \label{eqn:model}
    Y = X w^* + \xi
    ,\qquad
    X_i \stackrel{iid}{\sim} N(0, \Sigma)
    ,\qquad
    \xi \sim N(0, \sigma^2 I_n)
,\end{equation}
where $X \in \R^{n \times d}$ has i.i.d.\ Gaussian rows $X_1,\ldots,X_n \in \R^d$, $w^*$ is arbitrary,
and $\xi$ is Gaussian and independent of $X$.  The \emph{empirical} and \emph{population loss} are defined as, respectively,
\[
    \hat{L}(w) = \frac{1}{n} \Norm{Y - X w}_2^2
    ,\qquad
    L(w) = \E_{(x,y)} (y - \langle w, x \rangle)^2 = \sigma^2 + \Norm{w - w^*}_{\Sigma}^2
,\]
where in the expectation $y = \langle x, w^* \rangle + \xi_0$ with $x \sim N(0,\Sigma)$ independent of $\xi_0 \sim N(0,\sigma^2)$. When $d < n$, there is an unique minimizer of $\hat{L}$ which is the ordinary least square estimator $\wols = (X^TX)^{-1}X^TY$. When $d \ge n$, for an arbitrary norm $\norm{\cdot}$, the minimal norm interpolator is $\hat{w} = \argmin_{\hat{L}(w) = 0} \norm w$. 

\section{Optimistic Rates Theory} \label{sec:optimistic-rate}

As discussed by \citet{junk-feats}, a promising version of localized uniform convergence is to use bounds with ``optimistic rates'' \citep{panchenkooptimistic,srebro2010optimistic}, which establish different generalization guarantees depending on the size of the training error. (This broad concept has been studied at least since the work of \citet[Theorem 6.3]{vapnik2006estimation}.) In particular, \citet{srebro2010optimistic} show that with high probability, it holds uniformly over all $w \in \cH$ that
\begin{equation} \label{eqn:old-optimistic-rate}
    L(w) - \hat{L}(w) \leq \tilde{\mathcal O} \left( \sqrt{\hat{L}(w) \cdot \cR^2_n(\cH)} + \cR^2_n(\cH) \right)
\end{equation}
where $\cR_n(\cH)$ is the Rademacher complexity\footnote{\citet{srebro2010optimistic} consider the worst-case Rademacher complexity. In our results, we use a smaller quantity known as the average Rademacher complexity. The formal definition is given in \cref{sec:optimistic-rate-rademacher}.} of $\cH$ for any $n \in \N$. Considering only interpolators in $\cH$, the points for which $\hat{L}(w) = 0$), this bound becomes
\begin{equation} \label{eqn:old-interpolator-rate}
L(w) \leq \tilde{\mathcal O} \left( \cR^2_n(\cH) \right).
\end{equation}

In classical settings, it is typically the case that $\cR_n(\cH) \leq \sqrt{R/n}$ for some constant $R > 0$, and so \eqref{eqn:old-optimistic-rate} implies a graceful degradation from a learning rate of $\tilde{\mathcal O}(1/n)$ in realizable settings to a learning rate of $\tilde{\mathcal O}(1/\sqrt{n})$ in the more general non-realizable settings. The hidden constant and log factor in the $\tilde{\mathcal O}$ notation are not so problematic in this regime, because the quantity inside is vanishing.

In interpolation learning, however, we no longer have the scaling of $\cR_n(\cH) \leq \sqrt{R/n}$: the complexity required to perfectly fit the noisy observations needs to scale with the sample size,
and \citet{junk-feats} show in some cases that we can expect $\cR_n^2(\cH)$ to be approximately as large as the Bayes risk $\sigma^2$.
Therefore, any hidden factor greater than 1 inside the $\tilde{\mathcal O}$ notation of \eqref{eqn:old-interpolator-rate} will not be tight enough to establish consistency. In this work, we improve the hidden factor of $200\,000 \log^3\!n$ from \citet{srebro2010optimistic} to exactly 1, in the particular setting of Gaussian linear regression. Ignoring lower-order terms, we show that with high probability, the following inequality is approximately true for all $w \in \cH$:
\[
    L(w) - \hat{L}(w) \leq 2 \sqrt{\hat{L}(w) \cdot \cR^2_n(\cH)} + \cR^2_n(\cH),
\]
which can be more elegantly written as
\begin{equation} \label{eqn:new-optimistic-rate}
    L(w) \leq \left( \sqrt{\hat{L}(w)} + \cR_n(\cH) \right)^2.
\end{equation}
The formal statement is given in \cref{thm:covariance-splitting}. It will be clear from our applications in \cref{sec:applications} that the constants in \eqref{eqn:new-optimistic-rate} are in fact tight, and that the bound allows us to get precise generalization bounds for minimal-norm interpolation as well as ridge and LASSO regression.

\subsection{Main Bound}

We now give our main result, which will be used in \cref{sec:optimistic-rate-rademacher} to obtain \eqref{eqn:new-optimistic-rate}.

\begin{restatable}{theorem}{optimistic} \label{thm:optimistic}
Under the model assumption in \eqref{eqn:model},
let $F : \mathbb{R}^d \to [0, \infty]$
be a continuous function such that for $x \sim N(0,\Sigma)$, with probability at least $1 - \delta'$, it holds uniformly over all $w \in \mathbb{R}^d$ that
\begin{equation} \label{eqn:F-definition}
    \langle w-w^*, \, x \rangle \le F(w).
\end{equation}
For any $\delta > 0$, assume $n \geq 196 \log(12/\delta)$. Then there exists $\beta_1 \leq 14 \sqrt{\frac{\log(12/\delta)}{n}}$ such that with probability at least $1- 2(\delta'+\delta) $, it holds uniformly over all $w \in \mathbb{R}^d$ that 
\begin{equation} \label{eqn:optimistic-1}
    L(w) \leq (1+\beta_1) \left( \sqrt{\hat L(w)} + \frac{F(w)}{\sqrt{n}} \right)^2.
\end{equation}

\vspace{1ex}
\end{restatable}

The full proof can be found in \cref{sec:proof-optimistic}; we briefly sketch the proof here.

\begin{proof}[Proof sketch of \cref{thm:optimistic}]
We do this via Gordon's Theorem (also known as the Gaussian Minmax Theorem; see \cref{thm:gmt}). It suffices to prove that
\[ 
\sup_w \, \sqrt{\frac{L(w)}{1 + \beta_1}} - \left(\sqrt{\hat{L}(w)} + \frac{F(w)}{\sqrt{n}} \right) \le 0. 
\]
Write $X = Z \Sigma^{1/2}$, where $Z$ is a matrix of standard Gaussian entries. By the definitions of $\hat{L}(w)$ and $Y$, we have
\begin{equation*} 
 \sup_w \, \sqrt{\frac{L(w)}{1 + \beta_1}} - \frac{1}{\sqrt{n}} \left( \| Y - Xw \|_2 + F(w) \right)
= \sup_w \inf_{\| \lambda \|_2 = 1} \, \sqrt{\frac{L(w)}{1 + \beta_1}} + \frac{1}{\sqrt{n}} \left( \langle \lambda, Z \Sigma^{1/2} (w-w^*) - \xi \rangle - F(w) \right).
\end{equation*}
The last expression is a max-min optimization with a random Gaussian matrix $Z$, so by Gordon's Theorem we can prove a high-probability upper bound on this quantity (the ``Primary Optimization'') by upper-bounding the following ``Auxiliary Optimization'' problem with standard Gaussian vectors $H \sim N(0,I_d)$ and $G \sim N(0,I_n)$:
\begin{align*}
\MoveEqLeft \sup_w \inf_{\| \lambda \|_2 = 1} \, \sqrt{\frac{L(w)}{1 + \beta_1}} + \frac{1}{\sqrt{n}} \left(  \| \lambda \|_2 \langle H, \Sigma^{1/2} (w-w^*) \rangle + \|\Sigma^{1/2} (w-w^*)\|_2 \langle G, \lambda \rangle - \langle \lambda, \xi \rangle - F(w) \right) \\
&= \sup_w \inf_{\| \lambda \|_2 = 1} \, \sqrt{\frac{L(w)}{1 + \beta_1}} + \frac{1}{\sqrt{n}} \left(  \langle H, \Sigma^{1/2} (w-w^*) \rangle +\langle G  \|\Sigma^{1/2} (w-w^*)\|_2 - \xi, \lambda \rangle  - F(w) \right) \\
&= \sup_w \, \left[ \sqrt{\frac{L(w)}{1 + \beta_1}} - \frac{1}{\sqrt{n}} \| G  \|\Sigma^{1/2} (w-w^*)\|_2 - \xi\|_2 \right] + \frac{1}{\sqrt{n}} \left[ \langle \Sigma^{1/2}H,  w-w^* \rangle - F(w) \right].
\end{align*}
The first term is negative with high probability, because 
\[ 
L(w) = \|\Sigma^{1/2}(w-w^*)\|_2^2  + \sigma^2; 
\]
since $G, \xi$ are approximately orthogonal, we have
\[
\| G \|\Sigma^{1/2}(w^* - w)\|_2 - \xi \|_2^2 \approx \| G\|_2^2 \|\Sigma^{1/2}(w^* - w)\|_2^2 + \|\xi\|_2^2 \approx n(\|\Sigma^{1/2}(w-w^*)\|_2^2  + \sigma^2).
\]
The $1+\beta_1$ terms accounts for the variations in $G$ and $\xi$. The second term is also negative with high probability by the fact that $\Sigma^{1/2} H \sim \mathcal{N}(0, \Sigma)$ and our definition of $F$.
\end{proof}

\subsection{Gaussian width/Rademacher Bound} \label{sec:optimistic-rate-rademacher}

Now we discuss how to recover the Rademacher bound \eqref{eqn:new-optimistic-rate} by choosing an $F$ to satisfy the criterion \eqref{eqn:F-definition}. In the context of our model assumption \eqref{eqn:model}, the average Rademacher complexity is given by the following:
\begin{restatable}{defn}{rademacher} \label{def:rademacher}
Given a positive semi-definite matrix $\Sigma$ and sample size $n \in \N$, the \emph{Rademacher complexity} of a hypothesis class $\cH$ is given by
\[
\cR_n(\cH) = \E_{\substack{x_1, ..., x_n \sim \mathcal{N}(0,\Sigma) \\ s \sim \text{Unif} (\{\pm 1\}^n) }} \left[ \, \sup_{h \in \cH} \, \left| \frac{1}{n} \sum_{i=1}^n s_i h(x_i) \right| \, \right].
\]
Rademacher complexity measures the ability of $\cH$ to fit random Rademacher noise $(\pm 1)$ on an average training set sampled from the ground truth distribution. For more background, see for example the work of \citet{srebro2010optimistic,bartlett2002rademacher,bartlett2005local,wainwright2019high}. %
\end{restatable}

A closely related geometric complexity measure is the Gaussian width \citep[see, e.g.,][]{bartlett2002rademacher,vershynin2018high}. The following definitions match the notation of \citet{uc-interpolators}.

\begin{restatable}{defn}{gwidthrad} \label{def:gwidth-rad}
The \emph{Gaussian width} and the \emph{radius} of a set $S \subset \mathbb{R}^d$ are
\[ W(S) := \E_{H \sim \mathcal{N}(0,I_d)} \sup_{s \in S} |\langle s, H \rangle| \quad \text{and} \quad  \rad(S) := \sup_{s \in S} \|s\|_2 .\]
We also define the notation
\[ W_{\Sigma}(S) := W(\Sigma^{1/2} S) \]
to represent the Gaussian width with respect to covariance matrix $\Sigma$.  
\end{restatable}

As it turns out, when the hypothesis class $\cH$ is linear, the Rademacher complexity is actually equivalent to Gaussian width (up to a scaling of $1/\sqrt{n}$). 

\begin{proposition} \label{prop:rad-gw-equivalent}
Let $\cK$ be an arbitrary subset of $\R^d$ and consider $\cH = \{x \mapsto \langle w, x \rangle : w \in \cK \}$. Then, for any positive semi-definite matrix $\Sigma$, it holds that
\begin{equation}
    \cR_n(\cH) = \frac{W_{\Sigma}(\cK)}{\sqrt{n}}.
\end{equation}
\end{proposition}

\begin{proof}
Observe that for $x_1, ..., x_n \sim \mathcal{N}(0,\Sigma)$ independent of $s \sim \text{Unif} (\{\pm 1\}^n)$, we have
$
\frac{1}{n} \sum_{i=1}^n s_i x_i \sim \mathcal{N} \left(0, \frac{1}{n}\Sigma \right)
$.
The rest just follows from definitions:
\begin{equation*}
    \begin{split}
        \cR_n(\cH) 
        &= \E_{\substack{x_1, ..., x_n \sim \mathcal{N}(0,\Sigma) \\ s \sim \text{Unif} (\{\pm 1\}^n) }} \left[ \, \sup_{w \in \cK} \, \left| \frac{1}{n} \sum_{i=1}^n s_i \langle w, x_i \rangle \right| \, \right] \\
        &= \E_{\substack{x_1, ..., x_n \sim \mathcal{N}(0,\Sigma) \\ s \sim \text{Unif} (\{\pm 1\}^n) }} \left[ \, \sup_{w \in \cK} \, \left| \big\langle w, \frac{1}{n} \sum_{i=1}^n s_i x_i \big\rangle \right| \, \right] = \E_{H \sim \mathcal{N}(0, I_d)} \left[ \, \sup_{w \in \cK} \, \left| \big\langle w, \tfrac{1}{\sqrt{n}} \Sigma^{\frac12} H \big\rangle \right| \, \right] \\
        &= n^{-1/2} W_{\Sigma}(\cK). \qedhere
    \end{split}
\end{equation*}
\end{proof}

Consequently, to prove \eqref{eqn:new-optimistic-rate}, we can replace Rademacher complexity with Gaussian width, and we can see that the definition of $F$ in \cref{thm:optimistic} is very related to Gaussian width. To get tighter upper bounds, we recall the definition of covariance splitting \citep{uc-interpolators}, which is also used by \citet{bartlett2020benign}:

\begin{restatable}[Covariance splitting]{defn}{covsplit}
Given a positive semidefinite matrix $\Sigma \in \R^{d \times d}$,
we write $\Sigma = \Sigma_1 \oplus \Sigma_2$ if
$\Sigma = \Sigma_1 + \Sigma_2$,
each matrix is positive semidefinite,
and their spans are orthogonal.
\end{restatable}

To satisfy the definition of $F$ in condition \eqref{eqn:F-definition}, we can write $x = \Sigma^{1/2} H$, where $H \sim N(0,I_d)$. For any splitting $\Sigma = \Sigma_1 \oplus \Sigma_2$, let $H_1$ be the orthogonal projection of $H$ onto the span of $\Sigma_1$, and $H_2$ that onto the span of $\Sigma_2$.

\begin{example}[Gaussian width and \cref{thm:optimistic}]
If we are only interested in predictors from a fixed hypothesis class $\cK$, then by orthogonality, it holds that for all $w \in \cK$,
\begin{equation*}
    \begin{split}
        \langle w^* - w, \, x \rangle 
        &= \langle w^* - w, \, \Sigma^{1/2}_1 H \rangle + \langle w^* - w, \, \Sigma^{1/2}_2 H \rangle \\
        &= \langle w^* - w, \, \Sigma^{1/2}_1 H_1 \rangle + \langle w^* - w, \, \Sigma^{1/2}_2 H_2 \rangle \\
        &\leq \| \Sigma^{1/2} (w-w^*) \|_2 \cdot \| H_1 \|_2  + |\langle \Sigma_2^{1/2} w^*, H_2 \rangle| + \sup_{w \in \Sigma_2^{1/2}\cK} | \langle w, H_2 \rangle|.
    \end{split}
\end{equation*}
Hence, by standard concentration results and the fact that $\| \Sigma^{1/2} (w-w^*) \|_2 = \sqrt{L(w)-\sigma^2}$, we can choose 
\begin{equation*}
    F(w) = \left( \sqrt{\rank \Sigma_1} + 2\sqrt{\log(16/\delta')} \right) \sqrt{L(w) - \sigma^2} + W_{\Sigma_2}(\cK) + \left( \rad(\Sigma_2^{1/2} \cK) + \norm{w^*}_{\Sigma_2} \right) \sqrt{2 \log(16/\delta')}\\
\end{equation*}
for $w \in \cK$, and let $F(w) = \infty$ for $w \notin \cK$. 
\end{example}

Plugging into \cref{thm:optimistic} and rearranging the $\sqrt{L(w) - \sigma^2}$ term, we obtain the following:

\begin{restatable}{theorem}{CovSplitGen} \label{thm:covariance-splitting}
Under the model assumptions in \eqref{eqn:model}, let $\cK$ be an arbitrary compact set, and
take any covariance splitting $\Sigma = \Sigma_1 \oplus \Sigma_2$.
Fixing $\delta \le 1/4$, let $\beta_2 = 32 \left(\sqrt{\frac{\log(1/\delta)}{n}} + \sqrt{\frac{\rank(\Sigma_1) }{n}} \right)$.
If $n$ is large enough that $\beta_2 \leq 1$, then the following holds with probability at least $1 - \delta$ for all $w \in \cK$:
\begin{equation} \label{eqn:cov-split-weak}
    L(w) \leq (1+\beta_2) \left( \sqrt{\hat L(w)} + \frac{W_{\Sigma_2}(\cK)}{\sqrt{n}}+ \left[\|w^*\|_{\Sigma_2} + \rad(\Sigma_2^{1/2} \cK) \right] \sqrt{\frac{2\log(32/\delta)}{n}} \right)^2.
\end{equation}
Moreover, a stronger version of the above is also true: it holds that uniformly over all dilation factors $\alpha \ge 0$ and $w \in \alpha \cK$, we have
\begin{equation} \label{eqn:cov-split-strong}
    L(w) \leq (1+\beta_2) \left( \sqrt{\hat L(w)} + \frac{\alpha W_{\Sigma_2}( \cK)}{\sqrt{n}}+ \left[\|w^*\|_{\Sigma_2} + \alpha \rad(\Sigma_2^{1/2} \cK) \right] \sqrt{\frac{2\log(32/\delta)}{n}} \right)^2.
\end{equation}
\end{restatable}

The full proof can be found in \cref{sec:proof-optimistic}.
As discussed by \citet{uc-interpolators}, we can usually find a split such that the $\sqrt{\log(32/\delta)/n}$ term is negligible compared to the Gaussian width term, and so ignoring lower-order terms, our \cref{eqn:cov-split-weak} basically shows that 
\[
L(w) \leq \left( \sqrt{\hat L(w)} + \frac{W_{\Sigma_2}(\cK)}{\sqrt{n}} \right)^2,
\]
which, in light of \cref{prop:rad-gw-equivalent}, is the same as \eqref{eqn:new-optimistic-rate}. In addition, our stronger bound \eqref{eqn:cov-split-strong} shows that for any predictor $w$ outside $\cK$, we can always dilate $\cK$ by $\alpha$ and the Gaussian width term inside the corresponding upper bound will also be scaled by $\alpha$. Since our guarantee is uniform over $\alpha$, we are able to adapt our upper bounds to predictors with different norms and training errors at the same time. This will be useful for our applications in \cref{sec:optimistic-rate-flatness}, where we prove uniform generalization guarantees for all predictors along the regularization path.

\subsection{Special Case: Uniform Convergence of Interpolators}
\label{sec:optimistic-rate-uc-interpolators}

If we only look at interpolators in the set $\cK$, we immediately recover the uniform convergence of interpolators guarantee from \eqref{eqn:cov-split-weak}:

\begin{corollary}[Theorem 1 of \cite{uc-interpolators}] \label{cor:old-uc-interpolators}
Under the assumptions of \cref{thm:covariance-splitting}, we have with probability at least $1 - \delta$ that
\begin{equation}
    \sup_{w \in \cK, \hat{L}(w) = 0} L(w) \le \frac{1+\beta_2}{n} \left[ W_{\Sigma_2}(\cK) + \left[\|w^*\|_{\Sigma_2} + \rad(\Sigma_2^{1/2} \cK) \right]\sqrt{2\log\left( \frac{32}{\delta}\right)} \right]^2.
\end{equation}
\end{corollary}

It was shown that the above result can be used to tightly characterize the population risk of interpolating predictors. In particular, when the set $\cK = \{ w \in \R^d: \| w\| \leq B \}$ is a norm ball for some arbitrary choice of norm $\| \cdot \|$ and $B >0$, then the Gaussian width is 
\[
W_{\Sigma}(\cK) = B \cdot \E \| x\|_*
\]
where $\| \cdot \|_*$ is the dual norm and $x \sim \mathcal{N}(0,\Sigma)$. For example, if we consider the minimal-norm interpolator $\hat{w} = \argmin_{w: \hat{L}(w) = 0} \| w\|$ and choose $B$ to be a high probability upper bound of $\| \hat{w} \|$, then we approximately have
\begin{equation}
    L(\hat{w}) \leq \left(1 + o(1) \right) \cdot \frac{B^2 \left( \E \| x\|_* \right)^2}{n}.
\end{equation}
Combined with a norm analysis, \citet{uc-interpolators} show that \cref{cor:old-uc-interpolators} can recover the nearly-matching necessary and sufficient conditions from \citet{bartlett2020benign} for the consistency of the minimal $\ell_2$ norm interpolator. In particular, they show that
\[
B^2 \approx \sigma^2 \frac{n}{\left( \E \| x\|_* \right)^2}
\]
with lower-order terms depending on the effective ranks. In the context of $\ell_2$ penalty, recall the following definition of effective ranks:

\begin{restatable}[\cite{bartlett2020benign}]{defn}{euclideffrank} \label{def:ranks-l2}
The \emph{effective ranks} of a covariance matrix $\Sigma$ are
\[
r(\Sigma) = \frac{\Tr(\Sigma)}{\norm{\Sigma}_\OP} \quad \text{and} \quad R(\Sigma) = \frac{\Tr(\Sigma)^2}{\Tr(\Sigma^2)}
.\]
\end{restatable}

The lower-order terms will vanish when the $\ell_2$ benign overfitting conditions hold: there exists a sequence of covariance splits $\Sigma = \Sigma_1 \oplus \Sigma_2$ such that 
\begin{equation} \label{eqn:benign-overfitting-ridge}
    \frac{\rank(\Sigma_1)}{n} \to 0
    ,\qquad
    \norm{w^*}_2 \sqrt{\frac{\Tr(\Sigma_2)}{n}} \to 0
    ,\qquad
    \frac{n}{R(\Sigma_2)} \to 0.
\end{equation}
In this case, we have $L(\hat{w}) \to \sigma^2$ in probability with $\hat{w} = X^T(XX^T)^{-1}Y$ when $\| \cdot \|$ is the Euclidean norm, recovering in the Gaussian case the consistency result of \citet{bartlett2020benign,tsigler2020benign}. \citet{uc-interpolators} also demonstrate that \cref{cor:old-uc-interpolators} can establish the consistency of minimal-$\ell_1$ norm interpolators in certain settings, for which there are lower bounds that suggest the convergence rate from this analysis is nearly optimal \citep{chatterji2021foolish, muthukumar:interpolation}. We refer the reader to \citet{uc-interpolators} for details. 

\subsection{General Consequence: Flatness of Loss under Benign Overfitting Conditions} \label{sec:optimistic-rate-flatness}

In this section, we illustrate another consequence of \cref{thm:covariance-splitting} in the context of benign overfitting. As just discussed, even in situations where the labels have noise, there can be low-norm predictors that exactly interpolate the data and nevertheless generalize well. We see that our bounds from \cref{thm:covariance-splitting} and its special case \cref{cor:old-uc-interpolators} are sufficient to explain this phenomenon. In fact, they can tell us something more: the curve of the population loss along the regularization path will become flat in these settings, as long as the regularization parameter is small enough for us to obtain a predictor with norm larger than $\|w^*\|$. In other words, once we fit all of the signals, it does not matter how much noise is fitted, and all low norm near-interpolators can achieve consistency at the same time. 

In particular, if we take $\cK = \{w: \| w\|\leq 1\}$, then it is clear that for any $w \in \R^d$, we have $w \in \| w\| \cdot \cK$. To apply \eqref{eqn:cov-split-strong} of \cref{thm:covariance-splitting}, we define
\begin{equation} \label{eqn:C-definition}
    C_{\Sigma}(\| w\|) := \frac{\|w\| W_{\Sigma}(\cK)}{\sqrt{n}} +\left[\|w^*\|_{\Sigma} + \|w\| \rad(\Sigma^{1/2} \cK) \right] \sqrt{\frac{2\log(32/\delta)}{n}}.
\end{equation}

By virtue of \eqref{eqn:cov-split-strong}, if $w' \in \R^d$ (e.g., the minimal-norm interpolator) satisfies
\[
\hat{L}(w') = 0 \quad \text{ and } \quad C_{\Sigma_2} (\| w' \|) = \sigma + o(1),
\]
then $w'$ is a benign interpolator: $L(w') = \sigma^2 + o(1)$. Moreover, when the above holds, we can also establish consistency for any constrained empirical risk minimizer $\hat{w}_R$ of the form:
\begin{equation} \label{eqn:constrained-erm}
    \hat{w}_R := \argmin_{\|w\| \le R} \hat{L}(w)
\end{equation}
as long as $R$ is larger than $\| w^* \|$, and with the convention that if there are multiple minimizers then the minimum-norm minimizer is chosen.

\begin{restatable}{theorem}{Flatness} \label{thm:flatness-norm} 
Under the model assumptions in \eqref{eqn:model},
let $\|\cdot\|$ be an arbitrary norm on $\mathbb{R}^d$ and consider the complexity functional $C_{\Sigma}$ and the constrained ERM $\hat{w}_R$ given by \eqref{eqn:C-definition} and \eqref{eqn:constrained-erm}. Suppose there is a split $\Sigma = \Sigma_1 \oplus \Sigma_2$ and $\epsilon > 0$ such that with probability at least $1 - \delta$, it holds that 
\begin{equation} \label{eqn:flatness-w*-condition}
    \sqrt{\hat{L}(w^*)} \le (1 + \epsilon) \sigma 
    \quad \text{ and } \quad C_{\Sigma_2}(\| w^* \|) \leq \epsilon
\end{equation}
and there exists $w' \in \R^d$ such that
\begin{equation} \label{eqn:flatness-w'-condition}
    \hat{L}(w') = 0 \quad \text{ and } \quad C_{\Sigma_2}(\| w'\|) \le (1 + \epsilon) \sigma + \epsilon.
\end{equation}
Then, with probability at least $1-2\delta$, it holds uniformly over any $R \geq \|w^*\|$ that
\begin{equation} 
    L(\hat{w}_R) \le \left( \sigma+ 5 (\epsilon + \beta_2) (\sigma \vee 1) \right)^2.
\end{equation}
for the same choice of $\beta_2$ as in \cref{thm:covariance-splitting}.
\end{restatable}

The full proof, in \cref{sec:proof-optimistic}, follows based on a simple argument (\cref{lem:flatness}) which can be applied even more generally. The condition \eqref{eqn:flatness-w*-condition} can easily be satisfied using standard concentration results, whereas \eqref{eqn:flatness-w'-condition} requires some benign overfitting conditions. When there exists a benign interpolator, we can expect $\epsilon \to 0$ for a sufficiently large sample size, and so $L(\hat{w}_R)$ will converge to $\sigma^2$ uniformly. In the context of ridge regression ($\ell_2$ penalty), we want the condition \eqref{eqn:benign-overfitting-ridge} to hold. 
\begin{restatable}{corollary}{FlatnessRidge} \label{corr:ridge-flat}
Let $\sigma > 0$ be fixed. 
Under the assumptions of \cref{thm:flatness-norm} with $\| \cdot \|$ as the Euclidean norm, suppose that $\Sigma = \Sigma(n)$ is a sequence of covariance matrices with splits $\Sigma = \Sigma_1 \oplus \Sigma_2$ satisfying the benign overfitting conditions \eqref{eqn:benign-overfitting-ridge}. Then it holds that
\begin{equation}
    \sup_{R \geq \| w^* \|_2} L(\hat{w}_{R}) \to \sigma^2 \quad \text{in probability.}
\end{equation}
\end{restatable}
In other words, we get a uniform convergence result along this entire component of the regularization path. 
It is straightforward to make this into a finite-sample bound by using the non-asymptotic bounds on the norm of the minimum-norm interpolator from \citet{uc-interpolators}, as well as to generalize the result to other norms under the appropriate benign overfitting conditions from that work. We omit the details here.   

\begin{figure}
    \centering
    \includegraphics[scale=0.6]{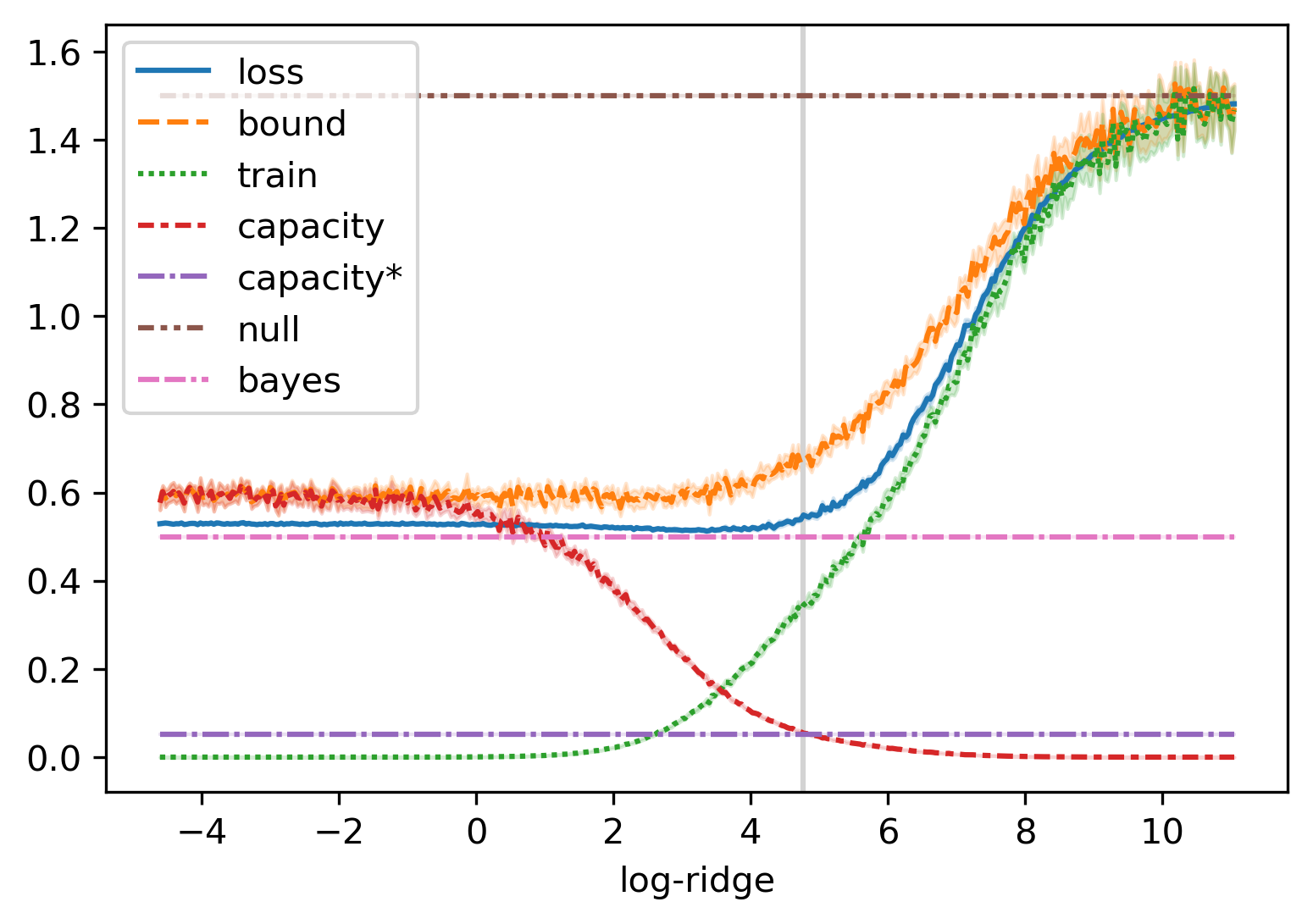}
    \caption{Loss along regularization path for ridge regression under benign overfitting conditions. Curve and error bars are computed from 10 trials with covariance matrix $\Sigma = \begin{bmatrix} 1 & 0 \\ 0 & \alpha^2 I_d \end{bmatrix}$, $\sigma^2 = 0.5$, $\alpha = 0.05$, and ground truth $w^* = (1,0,\ldots,0)$ from $n = 600$ samples with aspect ratio $d/n = 20$; the $x$-axis corresponds to the log of the ridge parameter. The curve ``bound'' corresponds to the generalization guarantee of \cref{thm:covariance-splitting}; it is close to the population loss (``loss'') along the whole regularization path. ``Null'' and ``bayes'' are $L(0)$ and $L(w^*)$. ``Capacity'' corresponds to the term $W(\mathcal{K})^2/n \approx \|w\|^2\Tr(\Sigma)/n$ for the ridge output $w$, and ``capacity\**'' is the same term with $\|w\|$ replaced by $\|w^*\|$. As predicted by \cref{thm:flatness-norm}, the population loss of the ridge regression is roughly flat once $\|w\| > \|w^*\|$ (threshold indicated by grey vertical line), and this is matched by the generalization bound, even though it is determined by the training error $\hat{L}(w)$ (curve ``train'') and capacity/norm $\|w\|$ which vary significantly.}
    \label{fig:my_label}
\end{figure}

\section{Applications}
\label{sec:applications}

In this section, we show how to apply our generalization bound to a variety of settings by choosing the appropriate complexity functional $F$ in \cref{thm:optimistic}, and by doing so we recover versions of classical results from compressed sensing, high-dimensional statistics, and statistical learning theory. Some aspects of our results are new: in particular, applying our theory always recovers finite-sample bounds and generally gives guarantees which apply to \emph{all predictors} in a class, not just the particular empirical risk minimizer. As further explained by \citet{junk-feats,uc-interpolators}, this is a crucial advantage of uniform-convergence based generalization bounds compared to other methods of analysis. For example, analyses based on random matrix theory methods or the asymptotic framework for applying the Convex Gaussian Minmax Theorem (CGMT) developed by \citet{thrampoulidis2015regularized,thrampoulidis2018precise} usually only give guarantees for the empirical risk minimizer and may have other limitations such as applying only in certain asymptotic limits. The key innovation here is \emph{not} that we can analyze convex M-estimators using Gordon's Theorem, which has indeed been done extensively in the literature, both in regularization and interpolation settings \citep[e.g.][]{rudelson2008sparse,chandrasekaran2012convex,stojnic2013framework,amelunxen2014living,deng2019model,oymak2010new,oymak2018universality,liang2020precise,raskutti2010restricted,montanari2019generalization} --- the point is the unifying power of the optimistic rates theory developed in the previous section, showing how many different phenomena can be understood from a simple and natural generalization theory approach. %

\subsection{Consistency of Optimally-tuned Regularized Regression}
To demonstrate the applicability of our \cref{thm:covariance-splitting} outside of the interpolation setting, we show how to apply it to derive consistency of optimally-tuned regularized least squares estimators such as the LASSO and Ridge regression. In particular, we will show the ridge estimator is consistent under a low effective dimension assumption on $\Sigma$; this kind of effective dimension condition was used, for example, by \citet{zhang2002effective,mendelson2003performance,tsigler2020benign}.

Given any predictor $w$, by the same reasoning in \cref{sec:optimistic-rate-uc-interpolators}, we obtain
\begin{equation} \label{eqn:uc-near-interpolators}
    L(w) \leq \left(1 + o(1) \right) \cdot \left( \sqrt{\hat{L}(w)} +  \frac{\| w\| \cdot \E \| x\|_*}{\sqrt{n}} \right)^2.
\end{equation}
For any $\lambda >0$, consider the regularized linear regression problem
\begin{equation} \label{eqn:regularized-regression}
    \hat{w}_{\lambda} = \argmin_w \, \hat{L}(w) + \lambda \| w \|.
\end{equation}
By comparing the KKT conditions, it is easy to see that there is some choice of $\lambda^*$ such that
\begin{equation*}
    \hat{w}_{\lambda^*} = \argmin_{\hat{L}(w) \leq \| \xi\|_2^2/n } \| w \|.
\end{equation*}
Since $\hat{L}(w^*) = \| \xi\|_2^2/n \approx \sigma$, it naturally follows that $\| \hat{w}_{\lambda^*} \| \leq \| w^* \|$. Plugging in the estimates into \eqref{eqn:uc-near-interpolators}, we obtain the following:

\begin{restatable}{corollary}{OptimallyTunedGen} \label{cor:optimally-tuned-gen}
Under the assumptions of \cref{thm:covariance-splitting},
consider the regularized regression estimators $\hat{w}_{\lambda}$ as in \eqref{eqn:regularized-regression} with an arbitrary norm $\| \cdot \|$. With probability at least $1-\delta$, there exists a $\lambda^* \geq 0$ such that
\begin{equation} \label{eqn:optimally-tuned-bound}
    L(\hat{w}_{\lambda^*}) \leq (1+3\beta_2) \left( \sigma + \frac{\| w^* \|}{\sqrt{n}} \left( \E_{x \sim \mathcal{N}(0, \Sigma_2)} \| x\|_{*} + \sup_{\| u\| \leq 1} \| u\|_{\Sigma_2} \cdot \sqrt{8\log(36/\delta)} \right) \right)^2.
\end{equation}
Hence, we have $L(\hat{w}_{\lambda^*}) \to \sigma^2$ in probability if 
\begin{equation} \label{eqn:optimally-tuned-condition}
    \frac{\rank(\Sigma_1)}{n} \to 0, \quad  \frac{\| w^* \| \cdot \E_{x \sim \mathcal{N}(0, \Sigma_2)} \| x\|_*}{\sqrt{n}} \to 0, \quad \text{and} \quad \frac{\| w^* \| \cdot \sup_{\| u\| \leq 1} \| u\|_{\Sigma_2}}{\sqrt{n}} \to 0.
\end{equation}
\end{restatable}

In the context of ridge regression, \eqref{eqn:optimally-tuned-bound} can be simplified to
\begin{equation}
    L(\hat{w}_{\lambda^*}) \leq (1+3\beta_2) \left( \sigma + \sqrt{32\log(36/\delta) \cdot \frac{\| w^* \|^2_2 \Tr(\Sigma_2)}{n}} \right)^2
\end{equation}
because both $\E_{x \sim \mathcal{N}(0, \Sigma_2)} \| x\|_2$ and $ \sup_{\| u\|_2 \leq 1} \| u\|_{\Sigma_2} = \| \Sigma_2 \|_{\OP}^{1/2}$ can be upper bounded by $\sqrt{\Tr(\Sigma_2)}$. Therefore, a sufficient condition for the consistency of optimally-tuned ridge regression is
\begin{equation}
    \frac{\rank(\Sigma_1)}{n} \to 0 \quad \text{ and } \quad \| w^*\|_2 \sqrt{\frac{\Tr(\Sigma_2)}{n}} \to 0.
\end{equation}

We see that the above is weaker than the benign overfitting condition \eqref{eqn:benign-overfitting-ridge} because we don't need the last condition $\frac{n}{R(\Sigma_2)} \to 0$. However, from \cref{sec:optimistic-rate-flatness}, having that condition means we no longer need to tune the ridge parameter $\lambda$: any sufficiently small $\lambda$ will lead to consistency.

\subsection{LASSO}

\paragraph{Slow Rate under Bounded $\ell_1$ Norm.}
In the context of LASSO regression, assume without loss of generality that the maximum diagonal entry of $\Sigma$ is $1$. Then we have
\[
\E_{x \sim \mathcal{N}(0, \Sigma_2)} \| x\|_{\infty} + \sup_{\| u\|_1 \leq 1} \| u\|_{\Sigma_2} \cdot \sqrt{8\log(36/\delta)} \lesssim \sqrt{\log(d)},
\]
and \eqref{eqn:optimally-tuned-bound} translates to the convergence rate of $\sigma \|w^*\|_1 \sqrt{\frac{\log(d)}{n}} + \|w^*\|_1^2 \cdot \frac{\log(d)}{n}$ to $\sigma^2$, which is also known as the ``slow'' rate of LASSO. Moreover, if
$w^*$ is $k$-sparse, then we can bound
\[
\| w^* \|_1 \leq k \|w^*\|_{\infty} 
\]
and so under these assumptions, the LASSO slow rate guarantee becomes $\sigma k \|w^*\|_{\infty} \sqrt{\frac{\log(d)}{n}} + k^2 \|w^*\|_{\infty}^2\cdot \frac{\log(d)}{n}$.
This analysis works for all predictors $w^*$ of bounded $\ell_1$-norm, and it is minimax optimal over this class, but when we assume that $w^*$ is $k$-sparse it is generally suboptimal and in particular does not give exact recovery when $\sigma = 0$. We now explain how our theory recovers the correct behavior in the sparse and well-conditioned setting commonly studied in the sparse linear regression literature. 

\paragraph{Performance under Sparsity and Compatability/Restricted Eigenvalue Condition.} 
We show how to recover well-known results from compressed sensing and high-dimensional statistics about sparse linear regression with Gaussian designs. In particular, we prove a performance guarantee for the LASSO when the covariance matrix is well-conditioned, as previously analyzed by \citet{raskutti2010restricted}, or more generally satisfies a version of the \emph{compatability condition} \citep{van2009conditions}. We start with the following well-known lemma commonly used in the analysis of the LASSO \citep[see, e.g.,][]{vershynin2018high}.

\begin{restatable}{lemma}{LASSOCone} \label{lem:cone}
Suppose $w^*$ is $k$-sparse, i.e. supported on coordinate set $S \subset [d]$ with $|S| \le k$. Every $w$ with $\|w\|_1 \le \|w^*\|_1$ satisfies
\begin{equation} \label{eqn:cone-inequality}
\|(w - w^*)_{S^C}\|_1 \le \|(w^* - w_S)\|_1. 
\end{equation}
\end{restatable}

The above lemma shows that the vector $w - w^*$ lies in the covex cone 
\[ \mathcal{C}(S) := \{ u : \|u_{S^C}\|_1 \le \|u_S\|_1 \}, \]
where $S$ is the support of $w^*$. Now we can state the version of the \emph{compatibility condition} \citep{van2009conditions} we use; the compatibility condition is a weakening of the \emph{restricted eigenvalue condition} \citep{bickel2009simultaneous,raskutti2010restricted}, and the compatibility condition is known to be a sufficient and almost necessary condition for the LASSO to perform exact recovery from $O(k\log d)$ samples in the Gaussian random design setting \citep{kelner2021power}.
\begin{defn}[Compatibility Condition; see \cite{van2009conditions}]\label{def:compatability-condition}
For a positive semidefinite matrix $\Sigma : n \times n$, $L \ge 1$, and set $S \subset [n]$, we say $\Sigma$ has \emph{$S$-restricted $\ell_1$-eigenvalue}
\[ \phi^2(\Sigma,S) = \min_{u \in \mathcal{C}(S)} \frac{|S| \cdot \langle u, \Sigma u \rangle}{\|u_S\|^2_1}.  \]
We say the \emph{S-compatibility condition} holds if the $S$-restricted $\ell_1$-eigenvalue is nonzero.
\end{defn}

\begin{example}[Application of \cref{thm:optimistic} to LASSO with sparsity]\label{ex:lasso}
Observe that for $x \sim N(0,\Sigma)$, we have by Holder's inequality,
the standard Gaussian tail bound, and the union bound that with probability at least $1 - \delta'$,
\begin{equation} \label{eqn:F-lasso-application}
\langle w - w^*, x \rangle \le \|w - w^*\|_1 \|x\|_{\infty} \le \|w - w^*\|_{1} \max_i \sqrt{2\Sigma_{ii} \log(2d/\delta')}.
\end{equation}
Thus, we can take $F(w)$ to be the right hand side of this inequality when applying \cref{thm:optimistic}.
\end{example}

Combining \eqref{eqn:F-lasso-application} with \cref{lem:cone} and the compatibility condition, we obtain the following:

\begin{restatable}{theorem}{LASSO} \label{thm:lasso-compatibility}
Under the model assumptions in \eqref{eqn:model}, additionally assume that:
\begin{enumerate}
    \item $w^*$ is a $k$-sparse vector.
    \item For $S \subset [d]$ the support of $w^*$, the covariance matrix $\Sigma$ satisfies the $S$-compatibility condition.
    \item The number of samples $n$ satisfies
    \[ n >  \frac{ 32 \max_i \Sigma_{ii} }{\phi^2(\Sigma,S)} \cdot k\log \left( \frac{32d}{\delta}\right). \]
\end{enumerate}
Then, for all $w$ satisfying $\|w\|_1 \le \|w^*\|_1$ and $\hat{L}(w) \le (1 + \epsilon)\sigma^2$ for an arbitrary $\epsilon$, we have
\begin{equation}
    L(w)-\sigma^2 \lesssim (\beta_1 + \epsilon)\sigma^2 + (1+\epsilon) \, \frac{ \max_i \Sigma_{ii}}{\phi(\Sigma,S)^2} \cdot \frac{\sigma^2 k \log(32d/\delta)}{n},
\end{equation}
where $\beta_1 = O(\sqrt{\log(1/\delta)/n})$ is as defined in \cref{thm:optimistic}. In particular, when $\sigma = 0$ we have that $\|w - w^*\|_{\Sigma} = 0$, and so if $\Sigma$ is positive definite then we have $w = w^*$ (exact recovery).
\end{restatable}

To interpret the above bound, observe that when we consider the ERM, we know that $\epsilon = O(1/\sqrt{n})$ based on concentration of the norm of the noise (\cref{lem:norm-concentration}) and so the first term is $\sigma^2/\sqrt{n}$ and the second term, assuming $\Sigma$ is well-conditioned, is $O(\sigma^2k\log(d/\delta)/n)$, which is the well-known minimax rate for sparse linear regression \citep[see, e.g.,][]{rigollet2015high}. The above analysis is not very careful in terms of constant factors; in \cref{sec:lasso-isotropic} we show how to get sharp constants in the isotropic setting. Also, in \cref{sec:low-complexity} we show how to get rid of the first term on the right hand side of the bound above, when we are specially considering the constrained ERM $\hat w$ minimizing the squared loss over all $\|w\|_1 \le \|w^*\|_1$, i.e. the LASSO solution: see \cref{corr:lasso-lowcomplexity}.

\subsection{Ordinary Least Squares}

Next, we consider a high-dimensional setting when $d$ is smaller than $n$. For example, when $d = n/2$, the ordinary least squares estimator $\wols$ is the unique minimizer of the training error, but it does not interpolate the training data and so the uniform convergence analysis of \citet{uc-interpolators} cannot be applied. As it turns out, our \cref{thm:optimistic} is enough to tightly characterize the excess risk of $\wols$.

\begin{example}[Application of \cref{thm:optimistic} to OLS] 
By the Cauchy-Schwarz inequality, it holds that
\begin{equation*}
    \langle w^* - w, \, x \rangle \leq \norm{H}_2 \Norm{w^* - w}_{\Sigma}.
\end{equation*}
Using standard concentration inequalities and $L(w) - \sigma^2 = \Norm{w - w^*}_{\Sigma}^2$, we can choose 
\begin{equation} \label{eqn:F-ols}
    F(w) = \left(\sqrt{d} + 2 \sqrt{\log(4/\delta')}\right) \sqrt{L(w)-\sigma^2}.
\end{equation}
\end{example}

\begin{restatable}{theorem}{OLSbound} \label{thm:OLS}
Under the model assumptions in \eqref{eqn:model}, let $\gamma = d/n < 1$.
There exists some $\epsilon \lesssim \left( \frac{\log(36/\delta)}{n}\right)^{1/2}$ such that for all sufficiently large $n$, with probability $1-\delta$ it holds uniformly for all $w \in \R^d$ that
\begin{equation} \label{eqn:ols-bound}
    \left| \sqrt{L(w)-\sigma^2} - \sqrt{\frac{\gamma\hat{L}(w)}{(1-\gamma)^2}} \, \right| \leq \epsilon \sqrt{\hat{L}(w)} + \sqrt{\frac{1}{1-\gamma} \left( \frac{\hat{L}(w)}{1-\gamma} - \sigma^2 \right) + \epsilon \hat{L}(w)}.
\end{equation}
For the empirical risk minimizer $\wols = (X^TX)^{-1}X^T Y$, the right hand side of \eqref{eqn:ols-bound} is approximately zero because we also have
\begin{equation}
    \hat{L}(\wols) \leq \sigma^2 (1-\gamma) + \sigma^2 \epsilon \sqrt{1-\gamma}.
\end{equation}
Therefore, we obtain the following generalization bound:
\begin{equation} \label{eqn:ols-bound2}
    L(\wols) - \frac{\sigma^2}{1-\gamma} \lesssim \sigma^2 \left( \frac{\log(36/\delta)}{n} \right)^{1/4}.
\end{equation}
\end{restatable}

We have a relatively complicated expression in \eqref{eqn:ols-bound} because our choice of $F$ according to \eqref{eqn:F-ols} depends on the excess risk $L(w)-\sigma^2$, and so after applying \eqref{eqn:optimistic-1} we need to solve a quadratic equation. All quantities in \eqref{eqn:ols-bound} are well-defined because $\hat{L} \geq 0$ and the $\epsilon \hat{L}(w)$ term inside the last square root ensures that with high probability it is positive. If we think of $\epsilon$ as zero for simplicity, then our uniform convergence guarantee \eqref{eqn:ols-bound} predicts that the excess risk $L(w) - \sigma^2$ of a predictor with training error $\hat{L}(w)$ cannot be larger than 
\[
\frac{1}{1-\gamma} \left( \sqrt{\frac{\gamma \hat{L}(w)}{1-\gamma} } + \sqrt{\frac{\hat{L}(w)}{1-\gamma} - \sigma^2} \right)^2.
\]
The minimal error is approximately $\sigma^2 (1-\gamma)$ and so all near empirical risk minimizer should enjoy an excess risk of $\sigma^2 \frac{\gamma}{1-\gamma}$, which agrees with the exact expectation formula in \citet{hastie2019surprises}; see their discussion for additional references. Since our approach also gives us a lower bound for free (by solving the quadratic equation), \cref{thm:OLS} is enough to show that $L(\wols)$ converges to $\sigma^2 \frac{1}{1-\gamma}$ in probability. We see that even though the empirical risk minimizer is not consistent, our localized uniform convergence approach can still provide an accurate understanding of the excess risk, and our bound for OLS is tight at least for the leading term. 

\begin{remark}\label{rmk:rates}
The $O(n^{-1/4})$ rate of \eqref{eqn:ols-bound2} comes from the fact that we need to take the square root of $\epsilon$ in the last term of \eqref{eqn:ols-bound}; it is not too difficult to see that this is sub-optimal for OLS. In fact, in \cref{thm:OLSvar}, we explicitly calculate the variance of $L(\wols)$ and show that in the proportional scaling regime (e.g., $\gamma = 0.5$), the right amount of deviation is of order $O(n^{-1/2})$. In the fixed-$d$ regime, the convergence rate can be accelerated to the more familiar rate of $O(n^{-1})$. In \cref{thm:OLShighprob}, we show how to use a more direct approach to obtain high probability bounds that match these variance calculations. Surprisingly, we can also show that the $O(n^{-1/4})$ rate is generally unavoidable for any uniform convergence analysis that only considers the size of $\hat{L}(w)$. Our analysis is tight in the sense that there are estimators whose training error is indistinguishable from $\wols$, but whose convergence rate is provably slower than $\Omega(n^{-1/4})$. For readers interested in the tightest rate of convergence, more details can be found in \cref{sec:improved-rate-ols}. 
\end{remark}

\subsection{Minimum-Euclidean Norm Interpolation with Isotropic Data and Proportional Scaling}

In the previous section, we saw that for OLS in the proportional scaling regime a simple application of our optimistic-rate bound recovers the limiting asymptotic population loss as a function of $d/n < 1$.
For $d/n > 1$, the OLS estimator is no longer defined, and instead we study the performance of the minimum-norm interpolator of the data. 
In \cref{thm:ridge-generalization} below, we show that with a slightly more careful\footnote{The specific choice of the complexity function $F$ follows from our \cref{lem:isotropic-bound} in the appendix.} application of \cref{thm:optimistic}, we can recover the loss curve at any aspect ratio (see \cref{fig:bd2}). Together with the previous result, we show that the optimistic-rate bound can capture the behavior of the pseudoinverse estimator $\hat w = X^+ Y$ on both sides of the double descent curve.  

\begin{restatable}{theorem}{RidgeGen} \label{thm:ridge-generalization}
Under the model assumptions in \eqref{eqn:model} with $\gamma = d/n > 1$ and $\Sigma = I_d$, there exists $\epsilon \lesssim \left(\frac{\log(18/\delta)}{n}\right)^{1/2}$ such that with probability at least $1-\delta$, the following holds uniformly over all $w$ such that $\hat{L}(w) = 0$:
\begin{equation}
    \left| L(w) - \left[ \sigma^2 + \|w\|_2^2 + \left( 1 -\frac{2 }{(1+\epsilon)\gamma}  \right) \|w^*\|_2^2 \right] \right| \leq 2\|w^*\|_2 \sqrt{\left( 1  -\frac{1 }{\gamma}  \right) \left( \|w\|_2^2 - \frac{\|w^*\|_2^2}{\gamma} \right)- \frac{\sigma^2}{\gamma} +  3\epsilon \| w \|_2^2}.
\end{equation}
\end{restatable}

\begin{figure}[tbh]
    \centering
    \includegraphics[scale=0.6,trim={0 0 0 0}]{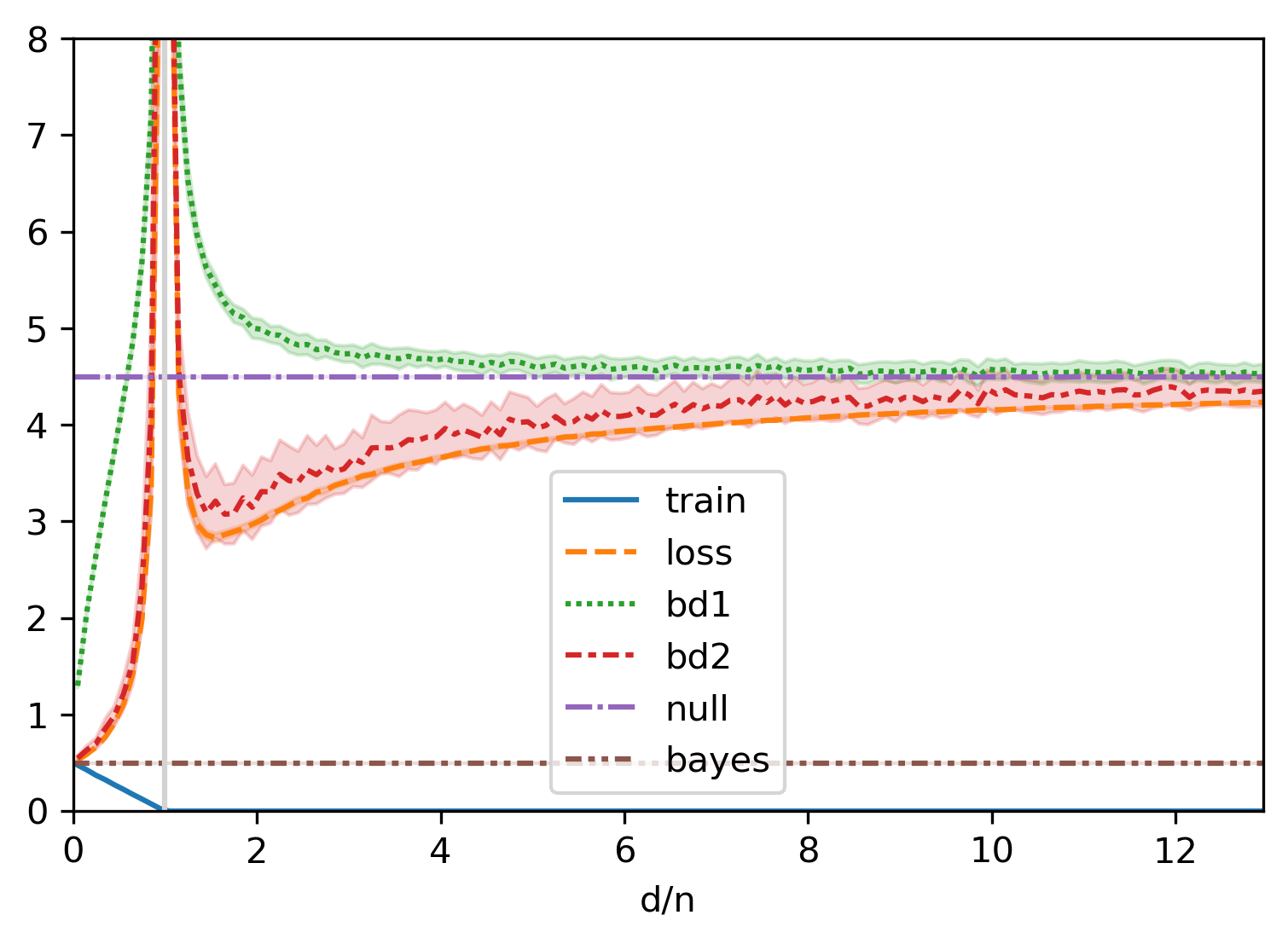}
    \caption{Generalization bounds for OLS/minimum-$\ell_2$ norm interpolation with isotropic covariance $\Sigma = I$, $\|w^*\| = 2$, $\sigma^2 = 0.5$, $n = 4096$, and varying aspect ratio $d/n$. The vertical line at $d/n = 1$ represents the double descent peak: on the left ($d/n < 1$) the predictor $w$ considered is the Ordinary Least Squares solution and on the right the minimum $\ell_2$-norm interpolator. The line ``train'' is the empirical loss $\hat{L}(w)$, the line ``loss'' is the test/population loss $L(w)$, ``bayes'' is the minimal population loss $L(w^*)$, and ``null' is $L(0)$.
    Each curve correspond to the means from 30 trials at each value of $d/n$, and the error bars correspond to standard deviations. 
    The line ``bd1'' corresponds to the bound $\left(\sqrt{\hat{L}(w)} + \|w\|\sqrt{d/n}\right)^2$ from \cref{thm:covariance-splitting}, and ``bd2'' is the upper bound from \cref{thm:OLS} for $d/n < 1$ and \cref{thm:ridge-generalization} for $d/n > 1$. As we see, bd2 is much closer to the true loss around the double descent peak.
    As explained in \cref{sec:local-gw}, bd2 can be recovered by looking at a localized version of Gaussian width.
    Both bd1 and bd2 are derived from our main optimistic rates bound \cref{thm:optimistic}.
    }
    \label{fig:bd2}
\end{figure}
It is clear from \cref{fig:bd2} below that \cref{thm:ridge-generalization} is capturing the asymptotic behavior of the minimum-norm interpolator; we prove this formally in \cref{thm:isotropic-norm} below by combining the generalization bound with a norm calculation, recovering the asymptotic formula for this setting computed by \citet{hastie2019surprises} using random matrix theory techniques. 

\begin{restatable}{theorem}{IsotropicNorm} \label{thm:isotropic-norm}
Under the model assumptions in \eqref{eqn:model} with $\gamma = d/n > 1$ and $\Sigma = I_d$, there exists $\epsilon \lesssim \left( \frac{\log(40/\delta)}{n}\right)^{1/2}$ such that with probability at least $1-\delta$, it holds that
\begin{equation}
    \min_{w: Xw = Y} \| w \|_2^2 \leq \left( 1 + \epsilon \right) \left( \frac{\| w^*\|_2^2}{\gamma} + \frac{\sigma^2}{\gamma-1} \right).
\end{equation}
Thus, by \cref{thm:ridge-generalization}, we have
\begin{equation}
    L(\hat{w}) - \left[  \left( 1 - \frac{1}{\gamma} \right) \| w^*\|_2^2 + \sigma^2 \frac{\gamma}{\gamma-1} \right] 
    \leq \epsilon \left(\frac{\| w^*\|_2^2}{\gamma} + \frac{\sigma^2}{\gamma-1} \right) + \| w^*\|_2 \sqrt{\epsilon \left(\frac{\| w^*\|_2^2}{\gamma} + \frac{\sigma^2}{\gamma-1} \right)} 
\end{equation}
where $\hat{w}$ is the minimal-$\ell_2$ norm interpolator. If we fix $\sigma^2, \gamma$ and $\| w^*\|_2$, then as $n \to \infty$
\begin{equation}
    L(\hat{w}) \to \left( 1 - \frac{1}{\gamma} \right) \| w^*\|_2^2 + \sigma^2 \frac{\gamma}{\gamma-1} \quad \text{ in probability.}
\end{equation}
\end{restatable}

\begin{remark}
Similar to the application in the last section, we also have a lower order $O(n^{-1/4})$ term. It is suboptimal, and we suspect that this is unavoidable for any uniform convergence analysis that only considers the typical size of $\| \hat{w}\|$. Nonetheless, this bound recovers the leading term, and the lower-order term is negligible if we only care about the difference with $\sigma^2$.
\end{remark}

\subsection{Sharp analysis of LASSO in the Isotropic Setting}\label{sec:lasso-isotropic}
A well-known application of the Gaussian Minmax Theorem is to the sharp analysis of the LASSO in the setting where the covariates are isotropic and Gaussian \citep[see, e.g.,][]{stojnic2013framework,amelunxen2014living}. Our optimistic rates bound \cref{thm:optimistic} recovers a corresponding generalization bound for all predictors $w$ with $\|w\|_1 \le \|w^*\|_1$, which when specialized to the constrained ERM (i.e.\ the LASSO solution) recovers these results.

\begin{restatable}{theorem}{LASSOIsotropic}\label{thm:lasso-isotropic}
Using the notation of \cref{thm:OLS}, we have with probability at least $1 - \delta$ that for all $w$ with $\|w\|_1 \le \|w^*\|_1$,
\begin{equation}
    \left| \sqrt{L(w)-\sigma^2} - \sqrt{\frac{\gamma\hat{L}(w)}{(1-\gamma)^2}} \, \right| \leq \epsilon \sqrt{\hat{L}(w)} + \sqrt{\frac{1}{1-\gamma} \left( \frac{\hat{L}(w)}{1-\gamma} - \sigma^2 \right) + \epsilon \hat{L}(w)}
\end{equation}
provided $\gamma + 2\epsilon/\sqrt{n} < 1$, where
\[ \cK' := \{u: \|w^* + u\|_1 \le \|w^*\|_1 \} \quad \text{ and } \quad \gamma := \frac{1}{n} \cdot W(\cK' \cap S^{n - 1})^2. \]
\end{restatable}

Observe that if $\sigma = 0$ and $\hat{L}(w) = 0$ then we get exact recovery provided $\gamma + 2\epsilon/\sqrt{n} < 1$ which is sharp up to the constant in the confidence term \citep[see, e.g.,][]{amelunxen2014living,chandrasekaran2012convex}. Informally, exact recovery occurs when $n > \omega^2$, i.e. the number of observations exceeds the statistical dimension. Moreover, we can consider the asymptotic setting where $\sigma = o(1)$ and the proportional scaling limit where $\gamma$ converges to constant. In this case, it is known  \citep[Equation 40(a)]{thrampoulidis2014gaussian} that we have
$\hat{L}(\hat w_{LASSO})/\sigma^2 \to 1 - \gamma$, so the right hand side of \eqref{eqn:lasso-bound} converges to zero and we have
\[ \frac{1}{\sigma^2} L(\hat{w}_{LASSO}) - 1 \to \frac{\gamma}{1 - \gamma}. \]
Thus we recover the characterization of the performance of LASSO in this regime \citep{thrampoulidis2014gaussian,stojnic2013framework}. It is possible, as in the OLS setting, to also derive non-asymptotic bounds on $\hat{L}(\hat w_{LASSO})$ and therefore obtain non-asymptotic bounds on the performance of the LASSO; we omit the details.

\begin{remark}
The Gaussian width of the tangent cone $\cK'$ has been sharply characterized in previous work \citep[e.g.][]{chandrasekaran2012convex,amelunxen2014living}.
In particular, from the work of \citet{amelunxen2014living} we know that if $w^*$ is $k$-sparse,
\[ \omega = W(\cK' \cap S^{n - 1}) \le W(\text{cone}(\cK') \cap S^{n - 1}) \le \sqrt{d \psi(s/d)} \]
where 
\[ \psi(\rho) := \inf_{\tau \ge 0} \left\{\rho (1 + \tau^2) + (1 - \rho) \sqrt{2/\pi} \int_{\tau}^{\infty} (u - \tau)^2 e^{-u^2/2} du\right\}, \]
as well as a corresponding lower bound which characterizes $\omega$. 
\end{remark}

\section{Localized Uniform Convergence Meets Localized Complexity Measure: the Optimality of Local Gaussian Width} \label{sec:local-gw}

Although our choice of the complexity function $F$ in the applications so far can seem quite mysterious, we show how it can be chosen systematically based on the regularizer or the geometry of the constraint set in this section. As we will see, the fact that we obtain the sharp constants in all of our analysis is not coincidental: the local Gaussian width theory can explain it and elucidate the connection to the previous asymptotic statistics literature (see \cref{rmk:comparison-moreau}).
Consider the following localized version of a convex set $\cK$:
\[ \cK_r := \{w \in \cK : \|w^* - w\|_{\Sigma} \le r \} \]
Based on \cref{prop:rad-gw-equivalent}, the corresponding Gaussian width $W_{\Sigma}(\cK_r)$ can be interpreted as a localized version of the Rademacher Complexity of the function class \citep[see, e.g.,][]{bartlett2005local,mendelson2014learning}). 

\paragraph{The optimal complexity functional.} Ignoring relatively minor technical issues involving the union concentration of Gaussian width, we can take $F(w) = W_{\Sigma}(\cK_{\|w - w^*\|_{\Sigma}})$ in the optimistic rates bound (\cref{thm:optimistic}). This choice of $F$ will lead to an optimal asymptotic guarantee in certain limits, particularly the proportional scaling limit. To see why, first note that if $r = \|w - w^*\|_{\Sigma}$, then we have from the optimistic rates bound that
\[ \sqrt{\sigma^2 + r^2} \le (1 + \beta_1)\left(\sqrt{\hat L(\hat w)} + W_{\Sigma}(\cK_r)/\sqrt{n}\right). \]
Rearranging and using $1/(1 + \beta_1) \ge 1 - \beta_1$ gives
\begin{equation}\label{eqn:rearranged-1}
 (1 - \beta_1)\sqrt{\sigma^2 + r^2} - W_{\Sigma}(\cK_r)/\sqrt{n} \le \sqrt{\hat{L}(\hat w)}. 
\end{equation}

For simplicity, denote the left hand side of \eqref{eqn:rearranged-1} as a function of $r$ called $\psi$. To obtain a learning guarantee in terms of $r$, we can find the sublevel set of $\psi$ based on the empirical loss. As the empirical loss becomes smaller, we will pick out a smaller and smaller sublevel set. When $\mathcal{K}$ is convex, it is just an interval because $\psi$ will be convex \footnote{For a proof, see \cref{lem:r-loss-convex} and \cref{lem:cKr-concave} in the appendix.}. On the other hand, we can use CGMT to analyze the minimal training error in $\cK$ and show that it nearly match the minimal value of $\psi$, see \cref{thm:erm-performance} below. This means that \eqref{eqn:rearranged-1} is nearly an equality for the ERM in $\cK$ and its excess risk $r$ is precisely determined by the minimizer of $\psi$. In applications, $\psi$ usually admits an unique minimizer, which confirms the approximate optimality of our generalization bound. We note that most of this discussion can also be generalized to non-convex sets $\cK$, but the minimal error in $\cK$ may no longer be determined by CGMT when $\cK$ is not convex.

We can now formalize this argument. First, we define two summary functionals similar to the left hand side of \eqref{eqn:rearranged-1}. 
For some absolute constant $C > 0$ and $\beta_1$ as defined in \cref{thm:optimistic}, we let the upper summary functional $\psi^+_{\delta}(x)$ at confidence level $\delta \in (0,1)$ to be
\begin{equation}\label{eqn:psi+}
\psi^+_{\delta}(r) := \max \left\{0, (1 + \beta_1)\sqrt{\sigma^2 + r^2} - W_{\Sigma}(\cK_r)/\sqrt{n} + C r\sqrt{\log(2/\delta)/n} \right\}
\end{equation}
and the lower summary functional $\psi^-_{\delta}(x)$ at confidence level $\delta \in (0,1)$ to be
\begin{equation}\label{eqn:psi-}
\psi^-_{\delta}(r) := \max\left\{0, (1 - \beta_1)\sqrt{\sigma^2 + r^2} - W_{\Sigma}(\cK_r)/\sqrt{n} - Cr\sqrt{\log(2/\delta)/n} \right\}. 
\end{equation}

The upper functional comes from the CGMT analysis of the minimal error while the lower functional comes from the application of \cref{thm:optimistic}. As discussed, they match except for a lower order term.

\begin{restatable}{theorem}{ERMPerformance}\label{thm:erm-performance}
Suppose that $\cK$ is a convex set and consider the upper summary function $\psi^+_{\delta}$ as defined in \eqref{eqn:psi+}. It holds with probability at least $1 - \delta$,
\begin{equation}
    \min_{w \in \cK}\sqrt{\hat{L}(w)} \le \min_{r \ge 0} \psi^+_{\delta}(r) 
\end{equation}
\end{restatable}

The following result, which is a formalization of \eqref{eqn:rearranged-1}, informally states that when a training error of $\mu^2$ is approximately achievable by any predictor in $\cK$, then only predictors $w$ with $\psi_{\delta}^-(\|w - w^*\|_{\Sigma}) \le \mu$ can achieve it --- note that by convexity, the set $\{r : \psi_{\delta}^-(r) \le \mu \}$ will always be an interval which shrinks as we decrease $\mu$. For the lower bound direction, the argument requires a union bound so we adjust the value of $\delta$ slightly to $\tau$; the difference is generally negligible since these confidence parameters only appear inside of logarithms.

\begin{restatable}{theorem}{LocalThm}\label{thm:local}
Suppose that $\cK$ is a convex set and consider the summary functional $\psi^+_{\delta}, \psi^-_{\delta}$ as defined in \eqref{eqn:psi+} and \eqref{eqn:psi-}. Let $\delta > 0$ and $\mu$ be arbitrary such that $\mu > \mu^* := \min_{r \ge 0}  \psi^+_{\delta}(r)$ and define $r^* := \inf \{ r : \psi^+_{\delta}(r) = \mu^* \}$.
Then with probability at least $1 - 4\delta$, it holds that uniformly over all $w \in \cK$ such that $\sqrt{\hat{L}(w)} \le \mu$ that:
\begin{equation}
\|w - w^*\|_{\Sigma} \le r_+ := \sup\{r \ge 0 : \psi^-_{\delta}(r) \le \mu \} 
\end{equation}
and also
\begin{equation}
\|w - w^*\|_{\Sigma} \ge r_- := \inf\left\{r \ge 0 : \psi^-_{\tau}(r) \le \mu \right\}
\end{equation}
where $\tau :=  \delta\big/\lceil \frac{\mu - \mu^*}{r^*} \rceil$. %
\end{restatable}

If we want to specifically analyze near-empirical risk minimizers, we can apply \cref{thm:local} with $\mu$ of the form $\mu^* + \epsilon$ with a small $\epsilon > 0$, and the conclusion is that their generalization error $\|w - w^*\|_{\Sigma}$ will be an approximate minimizer of the summary functional $\psi^-_{\delta}$. %

\begin{example}
To illustrate \cref{thm:local},
we briefly explain how to apply this result in the settings of OLS and minimum norm interpolation with isotropic data. Since we already have given precise nonasymptotic results for these settings in the previous sections, we only give a high-level summary of how to apply \cref{thm:local} in these examples and ignore, for example, the small difference between $\psi^-$,$\psi^+$ which is relevant for finite sample bounds. For OLS, we take $\cK = \mathbb{R}^d$ so $W_{\Sigma}(\cK_r) \approx r\sqrt{d/n}$ so the limiting summary functional is
\[ \psi(r) \approx \sqrt{\sigma^2 + r^2} - r\sqrt{d/n} \]
which is minimized at 
\[ r^2 = \sigma^2(d/n)/(1 - d/n), \]
so taking $\mu \to \psi(r)$ from above, we see by \cref{thm:local} %
that the OLS solution $\hat w$ satisfies $\|\hat w - w^*\|_{\Sigma} \in \psi^{-1}([0,\mu]) = \{\psi^{-1}(\mu)\} = \{r\}$ informally recovering the conclusion of \cref{thm:OLS}. For ridge regression (and in particular minimun norm interpolation) in the isotropic setting, we can reduce without loss of generality to the case where $\cK$ is the unit ball in which case $\cK_r$ is the intersection of the unit ball with a ball of radius $r$ about $w^*$: the Gaussian width of this intersection can be explicitly computed by solving a two-dimensional Euclidean geometry problem, and this essentially corresponds to the key \cref{lem:isotropic-bound} in the proof of \cref{thm:ridge-generalization}.
\end{example}
\begin{remark}[Comparison to Moreau Envelope Theory \citep{thrampoulidis2018precise}]\label{rmk:comparison-moreau}
In asymptotic settings where the two two summary functionals $\psi_{\delta}^-$ and $\psi_{\delta}^+$ both converge to a single limit $\psi$ with a unique minimizer,
\cref{thm:local} implies that the asymptotic error of the constrained empirical risk minimizer is given by the equation 
\[ \|\hat w - w^*\|_{\Sigma} = \arg\min_{r \ge 0} \psi(r). \]
In particular, the functional $\psi(r)$ serves as a ``summary functional'' which encapsulates all of the relevant information about the geometry of $w^*$ and $\cK$. 
In such an asymptotic setting, Theorem 3.1 of \citet{thrampoulidis2018precise} gives an asymptotic characterization of the performance of the constrained ERM (without any finite sample bounds) in terms of a summary functional called the ``expected Moreau envelope'': this can be understood as encoding almost the same information as $\psi(r)$. Some of the main advantages of \cref{thm:local} are that (1) it is nonasymptotic (in particular, it applies outside of the proportional scaling regime), 
(2) arguably easier to use and interpret, with a simple and direct connection to established notions of local complexity used in generalization theory \citep[see, e.g.,][]{bartlett2005local,mendelson2014learning}, and (3) it describes the generalization behavior of predictors $w$ besides the Empirical Risk Minimizer. Their result, while only applying in the proportional scaling limit, has the advantage of being applicable to other loss functions such as the Huber loss, being stated for more general noise models, and giving formulas directly in terms of regularization parameters without rewriting the optimization as a constrained optimization.
\end{remark}

\section{Improved finite-sample rate} \label{sec:improved-rate}
In this section, we discuss how to obtain improved finite sample rates and explain why the precise rates will depend on the particular information we have about the predictor. 
\subsection{Faster rates for low-complexity classes}\label{sec:low-complexity}
When the set $\cK$ is low complexity, as in the case of ordinary least squares when $d$ is fairly small compared to $n$, the optimal rate for the empirical risk minimizer in $\cK$ goes at a ``parametric rate'' of $1/n$, faster than a $1/\sqrt{n}$ rate. At first glance, it may appear impossible to get faster than a $1/\sqrt{n}$ rate from the main optimistic rates bound \cref{thm:optimistic} because of the presence of the $\beta_1 = O(\sqrt{\log(2/\delta)/n})$ term. As we will show, one can actually get fast/optimal rates from this theorem, but there is a different sense in which the $1/\sqrt{n}$ is unavoidable: this rate is actually the best we can hope for if we are only allowed to use certain summary statistics of the predictor (for example, see \cref{rmk:rates}).  
Nevertheless, it is still possible to obtain fast/optimal rates for the empirical risk minimizer by a black-box application of \cref{thm:optimistic}. The strategy we use is to bound the error $\|w - w^*\|_{\hat \Sigma}$ in the empirical metric by using a direct and very simple argument based on the KKT condition, and then apply \cref{thm:optimistic} to bound the error in the population metric. The general idea of analyzing the population loss by going through the empirical metric is very common in statistics and learning theory \citep[e.g.][]{mendelson2014learning,bartlett2006empirical,lecue2013learning}.

\begin{restatable}{theorem}{LowComplexity}\label{thm:low-complexity}
Let $\mathcal{K}$ be a closed convex set in $\mathbb{R}^d$ and $\delta' \ge 0, p \ge 0$ are such that with probability at least $1 - \delta'$ over the randomness of $x \sim N(0,\Sigma)$, uniformly over all $w \in \cK$ we have 
\begin{equation}
    \langle w - w^*, x \rangle \le \|w - w^*\|_{\Sigma} \sqrt{p}.
\end{equation}
Suppose that $\hat{w} = \argmin_{w \in \cK} \hat{L}(w)$ and $p/n \leq 0.999$, then for all $n \ge C\log(2/\delta)$ for some absolute constant $C > 0$, it holds with probability at least $1 - (\delta + \delta')$ that
\begin{equation}
    L(\hat{w}) - \sigma^2 \le (1 + \tau) \sigma^2 \cdot \frac{p}{n}.
\end{equation}
where $\tau = \tau(p,n,\delta)$ is upper bounded by an absolute constant and satisfies $\tau(p,n,\delta) \to 1$ in any joint limit $[p + \log(2/\delta)]/n \to 0, n \to \infty$.
\end{restatable}

The details of the proof can be found in  \cref{sec:proof-improved-rate}, where it is obtained as a special case of a more general result (\cref{thm:2bound}). To illustrate the application of this result, we show how it is used in the analysis of OLS. 

\begin{restatable}{corollary}{OLSFastRate} 
Under the model assumptions \eqref{eqn:model} with $d < n$ and assuming a sufficiently large $n$, it holds with probability at least $1-\delta$ that
\begin{equation}
    L(\wols) - \sigma^2 
    \lesssim \sigma^2 \left( \sqrt{\frac{d}{n}} + 2\sqrt{\frac{\log(36/\delta)}{n}}\right)^2
\end{equation} 
\end{restatable}

\cref{thm:low-complexity} can be applied in a very similar way to analyze other models in the low complexity regime, for example the LASSO when the sparsity level is small, which we illustrate below. Provided the $\ell_1$-eigenvalue $\varphi$ and maximum diagonal entry of $\Sigma$ are constants, we recover the sharp $\Theta(\sigma^2 k \log(d)/n)$ minimax rate for sparse linear regression (which is sharp provided $k \ll d$; see, e.g., \cite{rigollet2015high}). This recovers the guarantee for the LASSO in the Gaussian random design setting given by combining the result of \citet{raskutti2010restricted} with the appropriate analysis of LASSO in the fixed design setting \citep[e.g.][]{bickel2009simultaneous,van2009conditions}.

\begin{restatable}{corollary}{LASSOFastRate}\label{corr:lasso-lowcomplexity}
Applying \cref{thm:low-complexity} with $\cK = \{ \|w\|_1 \le \|w^*\|_1 \}$ the rescaled $\ell_1$-ball and under the sparsity and compatability condition assumptions of \cref{thm:lasso-compatibility}, we have with probability at least $1 - \delta$ that the LASSO solution
\[ \hat{w}_{LASSO} = \argmin_{w : \|w\|_1 \le \|w^*\|_1} \hat{L}(w) \]
satisfies
\begin{equation}
    L(\hat{w}_{LASSO}) - \sigma^2 \lesssim   \frac{\max_i\Sigma_{ii}}{\phi(\Sigma,S)^2} \cdot \frac{\sigma^2 k \log(16d/\delta)}{n}
\end{equation}
provided $n$ is sufficiently large that
\[ \sqrt{\frac{\max_i \Sigma_{ii}}{\phi(\Sigma,S)^2 } \cdot \frac{8k \log(16d/\delta)}{n}} \leq 0.999. \]
\end{restatable}

\subsection{Sharp Rate for OLS} \label{sec:improved-rate-ols}
We now zero in on the question of sharp rates for Ordinary Least Squares, returning to the discussion from \cref{rmk:rates}. Unlike all of the previous sections, in this section we will use tools beyond \cref{thm:optimistic} in order to precisely compute second order terms in the generalization gap. 
Surprisingly, even though we can match the high probability bound with an exact calculation up to first order term (see \cref{thm:OLS}), the existence of certain near-ERM can prevent us from recovering the correct variance term:

\begin{restatable}{theorem}{OLSlowerbd} \label{thm:OLSlowerbd}
Under the model assumptions in \eqref{eqn:model}, fix $\gamma = d/n$ to be some value in $(0,1)$ and pick any $c >0$. Then there exists another absolute constant $c^{\prime} > 0$ such that for all sufficiently large $n$, with probability at least $1-\delta$, there exists a $w \in \R^d$ such that 
\begin{equation}
    \hat{L}(w) - \hat{L}(\wols) \leq c \cdot \frac{\sigma^2}{n^{1/2}},
\end{equation}
but the population error satisfies 
\begin{equation}
    L(w) - L(\wols) \geq c^{\prime} \cdot \frac{\sigma^2}{n^{1/4}} .
\end{equation}
\end{restatable}

If we know that $\hat{L}(w) = \hat{L}(\wols) $, then it is necessarily the case that $w = \wols$ and as we will see, we can use \cref{thm:OLShighprob} to get the tightest possible convergence rates. On the other hand, it is not difficult to see that $n\hat{L}(\wols)/\sigma^2$ follows a chi-squared distribution with $n-d$ degrees of freedom, and by the variance formula of chi-squared distributions, we have
\[
\Var (\hat{L}(\wols)) = \frac{2\sigma^4(1-\gamma)}{n}.
\]
Consequently, $\hat{L}(\wols)$ can in fact deviate from $\E \hat{L}(\wols) = \sigma^2(1-\gamma)$ by the order of $\sigma^2/\sqrt{n}$. If we only know that $\hat{L}(w)$ is within the normal range of $\hat{L}(\wols)$, then the above theorem says that the sub-optimal rate of $O(n^{-1/4})$ that we show from \cref{thm:OLS} is actually tight and unavoidable. We can show a similar negative result for the fixed $d$ regime that the convergence cannot be faster than $O(n^{-1/2})$, but as we can see from the last section, using $\| w-w^*\|_{\hat{\Sigma}}^2 \approx \sigma^2 \gamma$ as the empirical metric instead is enough to recover the parameteric rate $O(1/n)$. This argument fails for the proportional limit regime because the smallest eigenvalue of $\hat{\Sigma}$ is $(1-\sqrt{\gamma})^2$ and so we can only get the larger quantity $\sigma^2 \frac{\gamma}{(1-\sqrt{\gamma})^2}$ which fails to capture the first order behavior of $\sigma^2 \frac{\gamma}{1-\gamma}$.

Finally, we show how to prove the tight finite sample rate using more direct methods. In fact, we can use the higher order moments of the inverse Wishart distribution \citep{vonRosen} to obtain the exact closed-form expressions for both the mean and variance of $L(\wols)$ with any finite value of $n$ and $d$. 

\begin{restatable}{theorem}{OLSvar} \label{thm:OLSvar}
Under the model assumptions in \eqref{eqn:model} with $d \leq n$, consider the ordinary least square estimator $\wols = (X^T X)^{-1}X^T Y$. It holds that
\begin{equation}
    \begin{split}
        \E L(\wols)  &= \sigma^2 \frac{n-1}{n-d-1} \\
        \Var (L(\wols)) &= 2\sigma^4 \frac{d(n-1)}{(n-d-1)^2(n-d-3)} \\
    \end{split}
\end{equation}
Hence as $d/n \to \gamma$, it holds that
\begin{equation}
    \E L(\wols) \to \frac{\sigma^2}{1-\gamma} \quad \text{and} \quad \frac{n}{\sigma^4}\Var (L(\wols)) \to  \frac{2\gamma}{(1-\gamma)^3}.
\end{equation}
If $d$ is held constant, as $n \to \infty$, we have
\begin{equation}
    n \E [L(\wols) - \sigma^2] \to \sigma^2 d \quad \text{and} \quad \frac{n^2}{\sigma^4} \Var (L(\wols)) \to 2 d.
\end{equation}
\end{restatable}

We can also show a matching high probability version of \cref{thm:OLSvar} based on the Gaussian minimax theorem:

\begin{restatable}{theorem}{OLShighprob} \label{thm:OLShighprob}
Under the model assumptions in \eqref{eqn:model} with $d \leq n$, consider the ordinary least square estimator $\wols = (X^TX)^{-1}X^TY$ and denote $\gamma = d/n.$ Assume that $\gamma \leq 0.999$, then with probability at least $1-\delta$, it holds that
\[
L(\wols) - \frac{\sigma^2}{1-\gamma} \lesssim \sigma^2 \sqrt{  \frac{\gamma \log(36/\delta)}{n}} .
\]
\end{restatable}

The full proof can be found in \cref{sec:proof-improved-rate}. As we can see from \cref{thm:OLSvar}, the variance of $L(\wols)$ is of order $O(1/\sqrt{n})$ when $d$ is proportional to $n$, and of order $O(1/n)$ when $d$ is fixed. In both cases, the expectation is close to $\sigma^2/(1-\gamma)$. \cref{thm:OLShighprob} shows exactly this and interpolates the two regimes: when $\gamma$ is of constant order, then we recover the $O(1/\sqrt{n})$ rate, but when $d$ is fixed, $\gamma = O(1/n)$ and so we can accelerate the convergence rate to $O(1/n)$. 

\begin{remark}
\citet{hastie2019surprises} provide a similar expectation calculation. On one hand, their results are more general in the sense that they do not assume the data is Gaussian, although the data is ``almost Gaussian'' because they require the existence of high-order moments. On the other hand, their results are asymptotic because their proof relies on the Marchenko-Pastur law and requires proportional scaling. In contrast, we obtain finite-sample bounds and to the best of our knowledge, we believe
that our variance calculation and high probability bounds are novel. 
\end{remark}

\section{Discussion}

In this work, we push the limit of what bounds with an optimistic rate can do. At least for well-specified linear regression with Gaussian data, we see that they are flexible enough to simultaneously understand interpolation learning and recover many classical results from compressed sensing, high dimensional statistics and learning theory. In the context of benign overfitting, not only can we establish the consistency of the minimal norm interpolator, we actually show that any predictor with a sufficiently low norm and training error can achieve consistency. In a variety of applications, we use our main theorem to obtain bounds with very sharp constants and our general theory suggests that we can always get a nearly optimal analysis for ERM in any convex set by choosing the complexity functional $F$ in \cref{thm:optimistic} based on local Gaussian width. 

A natural next step will be to relax the Gaussian assumption in our model \eqref{eqn:model} and also to consider situations where our linear model is misspecified in the sense that the Bayes optimal predictor is not linear. One of the key advantages of past works on uniform convergence, including the optimistic rate bound of \citet{srebro2010optimistic}, is that they do not need to make strong parameteric assumptions on the data distribution. Though the Gaussian width formulation of optimistic rate bounds, as in \eqref{eqn:cov-split-weak}, seems to crucially depend on the data being Gaussian, the connection to Rademacher complexity gives us hope that a version of our theory might apply to non-Gaussian data. (Some care must be taken in precisely formulating such a bound, due to the negative results discussed by \citet{foygel2011concentration,srebro2010optimistic}.) %
We also think that extending our results to generalized linear models, such as analyzing benign overfitting in linear classification, is an interesting direction. At least when the features are Gaussian, our techniques should be applicable; we leave this to future work.

\begin{refcontext}[sorting=nyt]
\printbibliography
\end{refcontext}

\pagebreak
\appendix
\section{Preliminaries}

\paragraph{Concentration of Lipschitz functions.} Recall that a function $f : \mathbb{R}^n \to \mathbb{R}$ is $L$-Lipschitz with respect to the norm $\norm\cdot$ if it holds for all $x, y \in \mathbb R^n$ that $|f(x) - f(y)| \le L\|x - y\|$. We use the concentration of Lipschitz functions of a Gaussian.

\begin{theorem}[\cite{van2014probability}, Theorem 3.25] \label{thm:gaussian-concentration}
If $f$ is $L$-Lipschitz with respect to the Euclidean norm and $Z \sim N(0,I_n)$, then
\begin{equation}
    \Pr(|f(Z) - \E f(Z)| \ge t) \le 2e^{-t^2/2L^2}.
\end{equation}
\end{theorem}

The proof of the following results can be found in \citet{uc-interpolators}.
\begin{lemma} \label{lem:norm-concentration}
Suppose that $Z \sim N(0,I_n)$. Then
\begin{equation}
    \Pr(\left|\|Z\|_2 - \sqrt{n}\right| \ge t) \le 4 e^{-t^2/4}.
\end{equation}
\end{lemma}

\begin{lemma} \label{lem:lowrank-projection}
Suppose that $S$ is a fixed subspace of dimension $d$ in $\mathbb R^n$ with $n \ge 4$, $P_S$ is the orthogonal projection onto $S$,
and $V$ is a spherically symmetric random vector (i.e. $V/\|V\|_2$ is uniform on the sphere). Then
\begin{equation}
    \frac{\|P_S V\|_2}{\|V\|_2} \le \sqrt{d/n} + 2\sqrt{\log(2/\delta)/n}.
\end{equation}

with probability at least $1 - \delta$. Conditional on this inequality holding, we therefore have uniformly for all $s \in S$ that
\begin{equation}  
|\langle s, V \rangle| = |\langle s, P_S V \rangle| \le \|s\|_2 \|P_S V\|_2 \le \|s\|_2\|V\|_2\left(\sqrt{d/n} + 2\sqrt{\log(2/\delta)/n})\right).
\end{equation}
\end{lemma}

\begin{theorem}[(Convex) Gaussian Minmax Theorem; \cite{thrampoulidis2015regularized,gordon1985some}]\label{thm:gmt}
Let $Z : n \times d$ be a matrix with i.i.d. $N(0,1)$ entries and suppose $G \sim N(0,I_n)$ and $H \sim N(0,I_d)$ are independent of $Z$ and each other. Let $S_w,S_u$ be compact sets and $\psi : S_w \times S_u \to \mathbb{R}$ be an arbitrary continuous function.
Define the \emph{Primary Optimization (PO)} problem
\begin{equation}
    \Phi(Z) := \min_{w \in S_w} \max_{u \in S_u} \langle u, Z w \rangle + \psi(w,u)
\end{equation}
and the \emph{Auxiliary Optimization (AO)} problem
\begin{equation}
    \phi(G,H) := \min_{w \in S_w} \max_{u \in S_u} \|w\|_2\langle G, u \rangle + \|u\|_2 \langle H, w \rangle + \psi(w,u).
\end{equation}
Under these assumptions, $\Pr(\Phi(Z) < c) \le 2 \Pr(\phi(G,H) \le c)$ for any $c \in \mathbb{R}$.

Furthermore, if we suppose that $S_w,S_u$ are convex sets and $\psi(w,u)$ is convex in $w$ and concave in $u$, then $\Pr(\Phi(Z) > c) \le 2 \Pr(\phi(G,H) \ge c)$. 
\end{theorem}

\section{Proofs for Section~\ref{sec:optimistic-rate}} \label{sec:proof-optimistic}

\subsection{Proof of Theorem~\ref{thm:optimistic}}

To apply the Gaussian Minimax Theorem, we first formulate the quantity of interest as an optimization problem in terms of a random matrix with $N(0,1)$ entries.

\begin{lemma} \label{lem:rewrite-gap}
Under the model assumptions in \eqref{eqn:model}, let $F$ be an arbitrary function and $\beta$ be any positive real number. Define the primary optimization problem (PO) as
\begin{equation} \label{eqn:phi}
    \Phi = \sup_w \, \inf_{\norm{\lambda}_2 = 1} \, \langle Zw, \lambda \rangle + \sqrt{\frac{1}{1+\beta} \left( n\sigma^2 + n\norm{w}_{2}^2 \right)} - \langle \xi, \lambda \rangle - F(\Sigma^{-1/2}w + w^*)
\end{equation}
where $Z$ is an $n \times d$ random matrix with i.i.d. standard normal entries independent of $\xi$ and each other. Then it holds that
\begin{equation}
    \sup_w \, \sqrt{\frac{1}{1+\beta} \cdot L(w)} - \left( \sqrt{\hat L(w)} + \frac{F(w)}{\sqrt{n}} \right) \overset{\cD}{=} \frac{1}{\sqrt{n}} \Phi
\end{equation}
\end{lemma}

\begin{proof}
By our definition of population and empirical loss, we have
\begin{equation*}
    \begin{split}
        &\sup_w \, \sqrt{\frac{1}{1+\beta} \cdot L(w)} - \left( \sqrt{\hat L(w)} + \frac{F(w)}{\sqrt{n}} \right) \\
        = \, & \sup_w \, \sqrt{\frac{1}{1+\beta} \left(\sigma^2 + \norm{w - w^* }_{\Sigma}^2 \right)} - \left(\frac{1}{\sqrt{n}} \norm{Y-Xw}_2 + \frac{F(w)}{\sqrt{n}} \right)  \\ 
        = \, & \sup_w \, \inf_{\norm{\lambda}_2 = 1} \, \sqrt{\frac{1}{1+\beta} \left(\sigma^2 + \norm{w - w^* }_{\Sigma}^2 \right)} - \left(\frac{1}{\sqrt{n}} \langle Y-Xw, \lambda \rangle + \frac{F(w)}{\sqrt{n}} \right)  \\ 
    \end{split}
\end{equation*}
By equality in distribution, we can write $X = Z \Sigma^{1/2}$. Using a change of variables, the above becomes
\begin{equation*}
    \begin{split}
        &\sup_w \, \inf_{\norm{\lambda}_2 = 1} \, \sqrt{\frac{1}{1+\beta} \left( \sigma^2 + \norm{w}_{2}^2 \right)} - \left(\frac{1}{\sqrt{n}} \langle \xi - Zw, \lambda \rangle + \frac{F(\Sigma^{-1/2}w + w^*)}{\sqrt{n}} \right) \\  
        = \, & \frac{1}{\sqrt{n}} \, \sup_w \, \inf_{\norm{\lambda}_2 = 1} \, \langle Zw, \lambda \rangle + \sqrt{\frac{1}{1+\beta} \left( n\sigma^2 + n\norm{w}_{2}^2 \right)} - \langle \xi, \lambda \rangle - F(\Sigma^{-1/2}w + w^*) \\
    \end{split}
\end{equation*}
\end{proof}

To apply \cref{thm:gmt}, we will use a truncation argument. The following result is an exercise in real analysis, which we include for completeness.
\begin{lemma} \label{lem:truncation}
Let $f: \R^d \to \R$ be an arbitrary function, then it holds that
\begin{equation}
    \lim_{r \to \infty} \sup_{\norm{w}_2 \leq r} f(w) = \sup_{w} f(w)
\end{equation}
\end{lemma}

\begin{proof}
We consider two cases:
\begin{enumerate}
    \item Suppose that $\sup_{w} f(w) = \infty$, then for any $M > 0$, there exists $x_M$ such that $f(x_M) > M$. Hence for any $r > \norm{x_M}_2$, it holds that
    \begin{equation*}
        \sup_{\norm{w}_2 \leq r} f(w) > M \implies \liminf_{r \to \infty} \sup_{\norm{w}_2 \leq r} f(w) \geq M
    \end{equation*}
    As the choice of $M$ is arbitrary, we have $\lim_{r \to \infty} \sup_{\norm{w}_2 \leq r} f(w) = \infty$ as desired. 
    
    \item Suppose that $\sup_{w} f(w) = M < \infty$, then for any $\epsilon > 0$, there exists $x_{\epsilon}$ such that $f(x_{\epsilon}) > M - \epsilon$. Hence for any $r > \norm{x_{\epsilon}}_2$, it holds that
    \begin{equation*}
        \sup_{\norm{w}_2 \leq r} f(w) > M - \epsilon \implies \liminf_{r \to \infty} \sup_{\norm{w}_2 \leq r} f(w) \geq M - \epsilon
    \end{equation*}
    As the choice of $\epsilon$ is arbitrary, we have $\liminf_{r \to \infty} \sup_{\norm{w}_2 \leq r} f(w) \geq M $. On the other hand, it must be the case (by definition of supremum) that 
    \begin{equation*}
        \sup_{\norm{w}_2 \leq r} f(w) \leq M \implies \limsup_{r \to \infty} \sup_{\norm{w}_2 \leq r} f(w) \leq M
    \end{equation*}
    Consequently, the limit of $\sup_{\norm{w}_2 \leq r} f(w)$ exists and equals $M$. \qedhere
\end{enumerate}
\end{proof}

\begin{lemma} \label{lem:gmt-app}
Let $G \sim N(0,I_n), H \sim N(0,I_d)$ be Gaussian vectors independent of $Z, \xi$ and each other. Define the auxiliary problem (AO) as
\begin{equation}
    \phi = \sup_w \, \sqrt{\frac{1}{1+\beta} \left( n\sigma^2 + n\norm{w}_{2}^2 \right)} - \norm{G \norm{w}_2  - \xi  }_2 + \langle H, w \rangle  - F(\Sigma^{-1/2}w + w^*).
\end{equation}
Suppose that $F$ is continuous, then it holds that for any $t \in \R$
\begin{equation}
    \Pr(\Phi > t \, | \, \xi) \le 2 \Pr(\phi \ge t \, | \, \xi),
\end{equation}
and taking expectations we have
\begin{equation}
    \Pr(\Phi > t ) \le 2 \Pr(\phi \ge t ).
\end{equation}
\end{lemma}

\begin{proof}
First, by \eqref{eqn:phi} define the truncated PO as
\begin{equation}
    \Phi_r = \sup_{\norm{w} \leq r} \, \inf_{\norm{\lambda}_2 = 1} \, \langle Zw, \lambda \rangle + \sqrt{\frac{1}{1+\beta} \left( n\sigma^2 + n\norm{w}_{2}^2 \right)} - \langle \xi, \lambda \rangle - F(\Sigma^{-1/2}w + w^*),
\end{equation}
and the corresponding AO is 
\begin{equation}
    \begin{split}
        \phi_r =& \sup_{\norm{w} \leq r} \, \inf_{\norm{\lambda}_2 = 1} \, \norm{w}_2 \langle G, \lambda \rangle + \norm{\lambda}_2 \langle H, w \rangle + \sqrt{\frac{1}{1+\beta} \left( n\sigma^2 + n\norm{w}_{2}^2 \right)} \\
        &\qquad \qquad \qquad - \langle \xi, \lambda \rangle - F(\Sigma^{-1/2}w + w^*) \\
        =& \sup_{\norm{w} \leq r} \, \langle H, w \rangle - \norm{G \norm{w}_2  - \xi  }_2 + \sqrt{\frac{1}{1+\beta} \left( n\sigma^2 + n\norm{w}_{2}^2 \right)} - F(\Sigma^{-1/2}w + w^*). \\
    \end{split}
\end{equation}
By \cref{lem:truncation}, with probability one, we have $\Phi_r$ and $\phi_r$ monotonically increase to $\Phi$ and $\phi$ as $r \to \infty$, respectively. By continuity of measure (from below), it holds that
\begin{equation*}
    \begin{split}
        \Pr (\Phi > t \, | \, \xi) &= \Pr \left(\lim_{r \to \infty} \Phi_r > t \, | \, \xi \right) \\
        &\leq \Pr \left( \cup_{r \in \N} \cap_{R \geq r} \Phi_R > t \, | \, \xi \right)\\
        &= \lim_{r \to \infty} \Pr \left( \cap_{R \geq r} \Phi_R > t \, | \, \xi \right) = \lim_{r \to \infty} \Pr \left( \Phi_r > t \, | \, \xi \right) \\
    \end{split}
\end{equation*}
By \cref{thm:gmt}, it follows that
\begin{equation*}
    \Pr \left( \Phi_r > t \, | \, \xi \right) = \Pr \left( -\Phi_r < -t \, | \, \xi \right) \leq 2 \Pr \left( -\phi_r < -t \, | \, \xi \right) \leq 2 \Pr (\phi > t \, | \, \xi)
\end{equation*}
Plugging in the bound above yields the desired conclusion. \qedhere
\end{proof}

\begin{lemma} \label{lem:aux-analysis}
Let $F$ satisfies the condition in \cref{thm:optimistic} and $n \geq 196 \log(12/\delta)$, then there exists $\beta \leq 14 \sqrt{\frac{\log(12/\delta)}{n}}$ such that
\begin{equation}
    \Pr(\phi \geq 0) \leq \delta' + \delta
\end{equation}
\end{lemma}

\begin{proof}
For notational simplicity, define
\begin{equation*}
    \begin{split}
        \alpha &:= 2\sqrt{\frac{\log(12/\delta)}{n}} \\
        \rho &:= \sqrt{\frac{1}{n}} + 2\sqrt{\frac{\log(6/\delta)}{n}}
    .\end{split}
\end{equation*}
By a union bound, the following collection of events occur with probability at least $1-\delta-\delta'$
\begin{enumerate}
    \item By \cref{lem:norm-concentration}, it holds that
        \begin{equation} \label{eqn:g-norm}
            1 - \alpha \leq \frac{1}{\sqrt{n}} \norm{G}_2
        \end{equation}
        and 
        \begin{equation} \label{eqn:xi-norm}
            1 - \alpha \leq \frac{1}{\sqrt{n} \sigma} \norm{\xi}_2
        \end{equation}
    \item By \cref{lem:lowrank-projection}, it holds that
        \begin{equation} \label{eqn:xi-g-orthogonal} 
            \langle \xi, G \rangle \leq \rho \norm{\xi}_2 \norm{G}_2
        \end{equation}
    \item By our assumption on $F$, it holds that uniformly over all $w \in \R^d$
        \begin{equation} \label{eqn:F-bounded}
            \langle H, w \rangle \leq F(\Sigma^{-1/2}w + w^*)
        \end{equation}
\end{enumerate}
Equations \eqref{eqn:g-norm}, \eqref{eqn:xi-norm} and \eqref{eqn:xi-g-orthogonal} implies that
\begin{equation*}
    \begin{split}
        \norm{G \norm{w}_2  - \xi  }_2^2 &\geq (1-\rho) \left( \norm{G}_2^2 \norm{w}_2^2 + \norm{\xi}_2^2 \right) \\
        &\geq (1-\rho) (1-\alpha)^2 n \left( \norm{w}_2^2 + \sigma^2 \right) \\
    \end{split}
\end{equation*}
Therefore, if we take $1+\beta = (1-\rho)^{-1} (1-\alpha)^{-2}$, combining with \eqref{eqn:F-bounded} shows that $\phi \leq 0$. To simplify the expression of $\beta$, observe that 
\begin{equation*}
    (1-\rho) (1-\alpha)^{2} \geq 1 - 2\alpha - \rho.
\end{equation*}
Finally, it is routine to check that $\beta \leq 14 \sqrt{\frac{\log(12/\delta)}{n}}$.
\end{proof}

\optimistic*

\begin{proof}
By \cref{lem:rewrite-gap} and \cref{lem:gmt-app}, we have
\begin{equation*}
    \begin{split}
        &\pr \left( \exists w \in \R^d, \, L(w) > (1+\beta) \left( \sqrt{\hat L(w)} + \frac{F(w)}{\sqrt{n}} \right)^2 \right) \\
        = & \pr \left( \sup_w \, \sqrt{\frac{1}{1+\beta} \cdot L(w)} - \left( \sqrt{\hat L(w)} + \frac{F(w)}{\sqrt{n}} \right) > 0 \right) \\
        = &\pr (\Phi > 0) \leq 2 \pr (\phi \geq 0)\\
    \end{split}
\end{equation*}
Finally, \cref{lem:aux-analysis} shows that $\Pr(\phi \geq 0) \leq \delta' + \delta$ and so the desired event occurs with probability at least $1-2(\delta' + \delta)$.
\end{proof}

\begin{remark}\label{rmk:convexity}
In \cref{thm:optimistic}, if the assumption \eqref{eqn:F-definition} is satisfied for a function $F$ then it is also satisfied for its greatest convex minorant $\text{conv}(F)$, which is the largest convex function such that $\text{conv}(F)(w) \le F(w)$ for all $w$, and replacing $F$ by $\text{conv}(F)$ only makes the conclusion stronger. Also, we note the conclusion can be written in terms of the population measure $\mu$ and empirical measure $\mu_n$ from $n$ samples as
\[ \|Y - \langle w, X \rangle\|_{L_2(\mu)} \le (1 + \beta)^{1/2} \left(\|Y - \langle w, X \rangle\|_{L_2(\mu_n)} + F(w)/\sqrt{n}\right) \]
so it can be interpreted as a lower isometry estimate for the empirical $L_2$ metric about the point $Y$. 
\end{remark}

\subsection{Proof of Theorem~\ref{thm:covariance-splitting}}

For convenience, we restate the theorem below: 
\CovSplitGen*

\begin{proof}
First, we show how to choose the complexity function $F$ in \cref{thm:optimistic} and show the result without dilations. 
We can write $x = \Sigma^{1/2} H$ where $H \sim N(0,I_d)$. For any splitting $\Sigma = \Sigma_1 \oplus \Sigma_2$, let $H_1$ be the orthogonal projection of $H$ onto the span of $\Sigma_1$. Similarly, we let $H_2$ be the orthogonal projection of $H$ onto the span of $\Sigma_2$. Then observe that\begin{equation*}
    \begin{split}
        \langle w^* - w, \, x \rangle 
        &= \langle w^* - w, \, \Sigma^{1/2}_1 H \rangle + \langle w^* - w, \, \Sigma^{1/2}_2 H \rangle \\
        &= \langle w^* - w, \, \Sigma^{1/2}_1 H_1 \rangle + \langle w^* - w, \, \Sigma^{1/2}_2 H_2 \rangle \\
        &\leq \| \Sigma_1^{1/2} (w-w^*) \|_2 \cdot \| H_1 \|_2  + |\langle \Sigma_2^{1/2} w^*, H_2 \rangle| + \sup_{w \in \Sigma_2^{1/2}\cK} | \langle w, H_2 \rangle| \\
    \end{split}
\end{equation*}
where the equality is by orthogonality of the split and the inequality is by Cauchy-Schwarz and the definition of supremum. Next, observe by \cref{lem:norm-concentration} that with probability at least $1 - \delta/8$, 
\[ \|H_1\| \le \sqrt{\rank \Sigma_1} + 2\sqrt{\log(32/\delta)}, \] 
and by \cref{thm:gaussian-concentration} with probability at least $1 - \delta/8$
\[ \sup_{w \in \Sigma_2^{1/2}\cK} | \langle w, H_2 \rangle| \le W_{\Sigma_2}(\cK) + \rad(\Sigma_2^{1/2} \cK)\sqrt{2\log(32/\delta)} \]
and by the standard Gaussian tail bound $\Pr_{Z \sim N(0,1)} (|Z| \geq t) \leq 2e^{-t^2/2}$, it holds that
\begin{equation}\label{eqn:signal-size} 
    |\langle \Sigma_2^{1/2} w^*, H \rangle| \le \|w^*\|_{\Sigma_2} \sqrt{2 \log(32/\delta)}
\end{equation}
because the marginal law of $\langle \Sigma_2^{1/2} w^*, H \rangle$ is $N(0, \|w^*\|_{\Sigma_2}^2)$.
Hence, by the union bound we have that with probability at least $1 - 3\delta/8$,
\[ \langle w^* - w, \, x \rangle \le F(w) := \|w^* - w\|_{\Sigma_1} \left( \sqrt{\rank \Sigma_1} + 2\sqrt{\log(32/\delta)} \right) + W_{\Sigma_2}(\cK) + [\|w^*\|_{\Sigma_2} + \rad(\Sigma_2^{1/2} \cK)]\sqrt{2\log(32/\delta)}. \]

Now applying \cref{thm:optimistic} with $F(w) = \infty$ outside of $\cK$ gives, where $\beta_1$ is as defined in the statement of that result,
\[ \sqrt{\frac{L(w)}{1+\beta_1}} \leq  \sqrt{\hat L(w)} +\|w^* - w\|_{\Sigma_1} \left( \sqrt{\frac{\rank \Sigma_1}{n}} + 2\sqrt{\frac{\log(32/\delta)}{n}} \right) + \frac{W_{\Sigma_2}(\cK)}{\sqrt{n}} + [\|w^*\|_{\Sigma_2} + \rad(\Sigma_2^{1/2} \cK)]\sqrt{\frac{2\log(32/\delta)}{n}} . \]

Observe that $\|w^* - w\|_{\Sigma_1} \le \|w^* - w\|_{\Sigma} \le \sqrt{L(w)}$ so we have
\[ \left( (1+\beta_1)^{-1/2} - \sqrt{\frac{\rank \Sigma_1}{n}} - 2\sqrt{\frac{\log(32/\delta)}{n}} \right) \sqrt{L(w)} \leq  \sqrt{\hat L(w)} + \frac{W_{\Sigma_2}(\cK)}{\sqrt{n}} + [\|w^*\|_{\Sigma_2} + \rad(\Sigma_2^{1/2} \cK)]\sqrt{\frac{2\log(32/\delta)}{n}}  \]
and by solving for $\sqrt{L(w)}$, we just need to consider $\beta_2$ such that
\[
\left( (1+\beta_1)^{-1/2} - \sqrt{\frac{\rank \Sigma_1}{n}} - 2\sqrt{\frac{\log(32/\delta)}{n}} \right)^{-2} \leq 1 + \beta_2.
\]
The above establishes the result when there is no dilation ($\alpha = 1$). Clearly, the same argument also shows the bound uniformly over all $\alpha \ge 0$ if we take 
\[ F(w) := \|w^* - w\|_{\Sigma_1} \left( \sqrt{\rank \Sigma_1} + 2\sqrt{\log(32/\delta)} \right) + \alpha(w) W_{\Sigma_2}(\cK) + [\|w^*\|_{\Sigma_2} + \alpha(w) \rad(\Sigma_2^{1/2} \cK)]\sqrt{2\log(32/\delta)} \] 
where $\alpha(w)$ is the infimum over all $\alpha$ such that $w \in \alpha \cK$.
\end{proof}

\subsection{Proof of Theorem~\ref{thm:flatness-norm}}

The following Lemma abstracts the key deterministic argument from the setting of \cref{thm:flatness-norm} to essentially any application of \cref{thm:optimistic}; the key insight is that a generalization bound of the form \eqref{eqn:flatness-bound} is exactly of the right form to explain flatness along the regularization path. Note that in the below Lemma, the function $F$ is assumed to be convex which is always without loss of generality when applying \cref{thm:optimistic}, see \cref{rmk:convexity}. 

\begin{lemma}\label{lem:flatness}
Suppose there exist a convex function $F$ and $\epsilon \in (0,1)$ such that:
\begin{enumerate}
    \item for all $w \in \mathbb{R}^d$, it holds that
    \begin{equation} \label{eqn:flatness-bound}
    \sqrt{L(w)} \le (1 + \epsilon) \left(\sqrt{\hat L(w)} + \frac{F(w)}{\sqrt{n}} \right). 
    \end{equation}
    \item $\epsilon$ is sufficiently large that
    \begin{equation} \label{eqn:flatness-w*}
        \sqrt{\hat{L}(w^*)} = \frac{\| \xi \|_2}{\sqrt{n}} \le (1 + \epsilon) \sigma 
        \quad \text{ and } \quad \frac{F(w^*)}{\sqrt{n}} \leq \epsilon.
    \end{equation}
    \item for some $w' \in \R^d$, it holds that 
    \begin{equation} \label{eqn:flatness-w'}
        \hat{L}(w') = 0 \quad \text{ and } \quad \frac{F(w')}{\sqrt{n}} \le (1 + \epsilon) \sigma + \epsilon.
    \end{equation}
\end{enumerate}
Then for all $R$ between $F(w^*)$ and $F(w')$ and any constrained empirical risk minimizer of the form 
\[ \hat{w}_{R} := \arg\min_{F(w) \le R} \hat{L}(w), \]
we have $L(\hat{w}_R) \le \left( \sigma+ 5\epsilon (\sigma \vee 1) \right)^2$.
\end{lemma}

\begin{proof}
For any $R$ between $F(w^*)$ and $F(w')$, we can write 
\[ R = (1 - \alpha) F(w^*) + \alpha F(w') \]
for some $\alpha \in [0,1]$. If we define $w_{\alpha} := (1 - \alpha) w^* + \alpha w'$ accordingly, then by convexity, we have
\[ F(w_{\alpha}) \le (1 - \alpha) F(w^*) + \alpha F(w') = R \]
and
\[ \hat{L}(w_{\alpha}) = \frac{1}{n} \|Y - X w_{\alpha}\|_2^2 = \frac{1}{n}(1 - \alpha)^2 \|Y - X w^*\|_2^2 = (1 - \alpha)^2 \hat{L}(w^*). \]
By the definition of $w_R$, it must be the case that $\hat{L}(w_R) \le \hat{L}(w_{\alpha})$ and so by \eqref{eqn:flatness-bound}, \eqref{eqn:flatness-w*} and \eqref{eqn:flatness-w'}
\begin{align*}
\sqrt{L(w_R)} 
&\le (1 + \epsilon) \left(\sqrt{\hat L(w_R)} + \frac{F(w_R)}{\sqrt{n}} \right) \\
&\le (1 + \epsilon) \left(\sqrt{\hat L(w_{\alpha})} + \frac{R}{\sqrt{n}} \right) \\
&= (1 + \epsilon) \left((1 - \alpha) \sqrt{\hat L(w^*)} +\frac{(1 - \alpha) F(w^*) + \alpha F(w')}{\sqrt{n}} \right) \\
&\leq (1 + \epsilon) \left((1 - \alpha) (1+\epsilon) \sigma + (1-\alpha) \epsilon + \alpha \left((1+\epsilon) \sigma + \epsilon \right) \right)  \\
&= (1+\epsilon)^2 \sigma + \epsilon(1 + \epsilon)\\
&\leq \sigma+ 5\epsilon (\sigma \vee 1). \qedhere
\end{align*}
\end{proof}

\Flatness*

\begin{proof}
Notice that $C_{\Sigma_2}$ is a monotone increasing function in $\| w \|$, so without loss of generality we can assume that $w'$ is the minimal norm interpolator. By \eqref{eqn:cov-split-strong} of \cref{thm:covariance-splitting} and condition \eqref{eqn:flatness-w'-condition}, we have
\[
L(w') \leq (1+\beta_2) \left( (1+\epsilon) \sigma + \epsilon \right)^2
\]
and it is easy to see this upper bound is no larger than the desired upper bound. By our convention, if $R > \| w' \|$ then $\hat{w}_R = w'$ and we are done. So we only need to consider the case when $\| w^* \| \leq R \leq \| w'\|$. 

To apply \cref{lem:flatness}, consider $F(w) = \sqrt{n}C_{\Sigma_2}(\| w\|)$ and let $\epsilon + \beta_2$ plays the role of $\epsilon$. Clearly, \eqref{eqn:flatness-bound}, \eqref{eqn:flatness-w*} and \eqref{eqn:flatness-w'} are satisfied by \cref{thm:covariance-splitting} and our assumptions \eqref{eqn:flatness-w*-condition} and \eqref{eqn:flatness-w'-condition}. Since $C_{\Sigma_2}$ is monotone increasing, the condition that $\| w\| \leq R$ is the same as $F(w) \leq \sqrt{n}C_{\Sigma_2}(R)$ and $F(w^*) \leq \sqrt{n}C_{\Sigma_2}(R) \leq F(w')$. We can conclude the proof by a union bound.
\end{proof}

\FlatnessRidge*

\begin{proof}
By \cref{lem:norm-concentration}, with probability at least $1-\delta/2$, we have $\sqrt{\hat{L}(w^*)} \leq \left(1 + 2 \sqrt{\frac{\log(8/\delta)}{n}}\right) \sigma$. Theorem 2 and 3 of \citet{uc-interpolators} shows that we can pick $w'$ to be the minimal $\ell_2$ norm interpolator, and there exists 
\[
\gamma \lesssim \sqrt{\frac{\log(1/\delta)}{n}}+ \sqrt{\frac{\log(1/\delta)}{r(\Sigma_2)}}  + \frac{n\log(1/\delta)}{R(\Sigma_2)} 
\]
such that with probability at least $1-\delta/2$, we have
\[
C_{\Sigma_2}(\| w'\|_2) \leq (1+\gamma) \left( \sigma + \| w^*\|_2 \sqrt{\frac{\Tr(\Sigma_2)}{n}}\right).
\]
So we can take $\epsilon$ to be the maximum of $2 \sqrt{\frac{\log(8/\delta)}{n}}$, $C_{\Sigma_2}(\| w^* \|_2)$, $\gamma$ and $(1+\gamma)\| w^*\|_2 \sqrt{\frac{\Tr(\Sigma_2)}{n}}$. We can apply \cref{thm:flatness-norm} and observe that $\epsilon + \beta_2 \to 0$ under the benign overfitting conditions \eqref{eqn:benign-overfitting-ridge}.
\end{proof}

\section{Proofs for Section~\ref{sec:applications}} \label{sec:proof-erm}

\subsection{Optimally-tuned regularized regression}
\OptimallyTunedGen*

\begin{proof}
By comparing the KKT conditions, it is easy to see that there is some choice of $\lambda^*$ such that
\begin{equation*}
    \hat{w}_{\lambda^*} = \argmin_{\hat{L}(w) \leq \| \xi\|^2/n } \| w \|.
\end{equation*}
Since $\hat{L}(w^*) = \| \xi\|^2/n$, it follows that $\| \hat{w}_{\lambda^*} \| \leq \| w^* \|$. To apply \cref{thm:covariance-splitting}, we consider $\cK = \{ w: \| w \| \leq 1\}$ and observe that 
\[
\rad(\Sigma_2^{1/2} \cK) = \sup_{\| w\| \leq 1} \| w\|_{\Sigma_2} \quad \text{ and } \quad \| w^* \|_{\Sigma_2} \leq \| w^* \| \cdot \sup_{\| w\| \leq 1} \| w\|_{\Sigma_2}.
\]
Plugging in \eqref{eqn:cov-split-strong}, by \cref{lem:norm-concentration} and a union bound,  we obtain
\begin{equation*}
    \begin{split}
        L(\hat{w}_{\lambda^*}) 
        &\leq (1+\beta_2) \left( \sqrt{\hat L(\hat{w}_{\lambda^*})} + \frac{\| \hat{w}_{\lambda^*} \| \cdot \E \| x\|_{*} }{\sqrt{n}}+ \left[\|w^*\|_{\Sigma_2} + \| \hat{w}_{\lambda^*} \| \cdot  \sup_{\| w\| \leq 1} \| w\|_{\Sigma_2} \right] \sqrt{\frac{2\log(36/\delta)}{n}} \right)^2 \\
        &\leq (1+\beta_2) \left( \frac{\| \xi \|_2}{\sqrt{n}} + \frac{\| w^* \| \cdot \E \| x\|_{*} }{\sqrt{n}}+ \| w^* \| \cdot  \sup_{\| w\| \leq 1} \| w\|_{\Sigma_2} \cdot \sqrt{\frac{8\log(36/\delta)}{n}} \right)^2 \\
        &\leq (1+\beta_2) \left( \left( 1 + 2 \sqrt{\frac{\log(36/\delta)}{n}}\right) \sigma + \frac{\| w^* \| \cdot \E \| x\|_{*} }{\sqrt{n}}+ \| w^* \| \cdot  \sup_{\| u\| \leq 1} \| u\|_{\Sigma_2} \cdot \sqrt{\frac{8\log(36/\delta)}{n}} \right)^2 \\
    \end{split}
\end{equation*}
It is routine to check that $(1+\beta_2)\left( 1 + 2 \sqrt{\frac{\log(36/\delta)}{n}}\right)^2 \leq 1 + 3\beta_2$ and the proof is complete.
\end{proof}

\subsection{LASSO}

\LASSOCone*

\begin{proof}
Note that over this set, we have 
\begin{equation*}
\|(w - w^*)_{S^C}\|_1 = \|w_{S^C}\|_1 = \|w\|_1 - \|w_S\|_1 \le \|w^*\|_1 - \|w_S\|_1 \le \|(w^* - w_S)\|_1 
\end{equation*}
where the first inequality uses $\|w\|_1 \le \|w^*\|_1$ and the second inequality is the triangle inequality. 
\end{proof}

\LASSO*

\begin{proof}
We start with the application of \cref{thm:optimistic} as in \cref{ex:lasso}.
Observe that for $x \sim N(0,\Sigma)$ we have by \cref{lem:cone}, the compatibility condition, the standard Gaussian tail bound and the union bound that with probability at least $1 - \delta/8$,
\begin{equation}\label{eqn:F-lasso}
    \begin{split}
        \langle w - w^*, x \rangle 
        &\le \|w - w^*\|_1 \|x\|_{\infty} \le 2 \|(w - w^*)_S\|_1 \|x\|_{\infty} \\
        &\le 2 \frac{k^{1/2} \|w - w^*\|_{\Sigma} }{\phi(\Sigma,S)} \max_i \sqrt{2\Sigma_{ii} \log(16d/\delta)} 
    \end{split}
\end{equation}
so applying \cref{thm:optimistic} with $F(w)$ equal to the right hand side of \eqref{eqn:F-lasso} gives
\begin{align*} 
\sigma^2 + \|w - w^*\|_{\Sigma}^2 = L(w) 
&\le (1 + \beta_1)\left(\sqrt{\hat{L}(w)} + \frac{2k^{1/2}}{\phi(\Sigma,S)} \|w - w^*\|_{\Sigma} \max_i \sqrt{2\Sigma_{ii} \log(32d/\delta)/n}\right)^2 \\
&\le (1 + \beta_1)\left(\sigma \sqrt{1 + \epsilon} + \frac{2k^{1/2}}{\phi(\Sigma,S)} \|w - w^*\|_{\Sigma} \max_i \sqrt{2\Sigma_{ii} \log(32d/\delta)/n}\right)^2
\end{align*}
For a sufficiently large $n$, we have $\beta_1 \leq 1$. Expanding the square and rearranging gives
\begin{equation*}
    \begin{split}
        \|w - w^*\|_{\Sigma}^2 
        &\le [\beta_1 + \epsilon + \epsilon\beta_1]\sigma^2 + 8 \sigma\frac{\sqrt{k(1 + \epsilon)}}{\phi(\Sigma,S)} \|w - w^*\|_{\Sigma} \max_i \sqrt{2\Sigma_{ii} \log(32d/\delta)/n} \\
        &\qquad + \frac{16 k\max_i \Sigma_{ii} \log(32d/\delta)}{\phi(\Sigma,S)^2} \cdot \frac{\|w - w^*\|_{\Sigma}^2}{n}
    \end{split}
\end{equation*}

and using the assumption on $n$ to rearrange the last term gives 
\begin{equation*}
    \begin{split}
        \|w - w^*\|_{\Sigma}^2 
        &\le 2[\beta_1 + \epsilon + \epsilon\beta_1]\sigma^2 + 16\sigma\frac{\sqrt{k(1 + \epsilon)}}{\phi(\Sigma,S)} \|w - w^*\|_{\Sigma} \max_i \sqrt{2\Sigma_{ii} \log(32d/\delta)/n} \\
        &\le 4[\beta_1 + \epsilon] \sigma^2 + \sqrt{\frac{512 \sigma^2 k(1 + \epsilon) \max_i  \Sigma_{ii} \log(32d/\delta)}{\phi(\Sigma,S)^2 n}} \cdot \|w - w^*\|_{\Sigma}. \\
    \end{split}
\end{equation*}

Solving this quadratic equation, it is not to difficult to check that
\[ \|w - w^*\|_{\Sigma}^2 \le 8[\beta_1 + \epsilon]\sigma^2 + \frac{512 (1 + \epsilon) \max_i \Sigma_{ii}}{\phi(\Sigma,S)^2} \frac{\sigma^2 k \log(32d/\delta)}{n} \]
which is the desired result. 
\end{proof}

\begin{remark}[Generalization Bound for Larger Cones]
For simplicity, in the above analysis we gave a generalization bound for predictors $w$ satisfying $\|w\|_1 \le \|w^*\|_1$, or more generally $\|(w - w^*)_{S^C}\|_1 \le \|(w - w^*)_S\|_1$, which covers the case of the LASSO with oracle regularization commonly considered in the literature \citep[see, e.g.,][]{vershynin2018high}. In situations where adaptivity to the unknown value of $\|w^*\|_1$ is important, the relevant predictor $w$ may only be guaranteed to satisfy the weaker bound $\|(w - w^*)_{S^C}\|_1 \le C \|(w - w^*)_S\|_1$ for some $C > 1$ and the analogous version of the compatibility condition/restricted eigenvalue condition over this cone is assumed \citep[see, e.g.,][]{bickel2009simultaneous,van2009conditions,rigollet2015high,wainwright2019high}; adopting the analysis to predictors in this larger cone is straightforward and we omit the details.  
\end{remark}

\subsection{OLS}

The following training error bounds are standard, which we include for completeness.
\begin{lemma} \label{lem:ols-train-err}
Under the model assumptions in \eqref{eqn:model} with $d \leq n$, consider the ordinary least square estimator $\wols = (X^T X)^{-1}X^T Y$. With probability at least $1-\delta$, it holds that
\begin{equation}
    \sqrt{\hat L(\wols)} \leq \sigma \left( \sqrt{1-\frac{d}{n}} + 2\sqrt{\frac{\log(4/\delta)}{n}}\right)
\end{equation}
Similarly, with probability at least $1-\delta$, it holds that
\begin{equation}
    \norm{\wols-w^*}_{\hat \Sigma} \leq \sigma \left( \sqrt{\frac{d}{n}} + 2\sqrt{\frac{\log(4/\delta)}{n}}\right)
\end{equation}
\end{lemma}

\begin{proof}
By our model assumptions, we can write $\wols = w^* + (X^T X)^{-1}X^T\xi$, and so $Y - X \wols = (I - X(X^T X)^{-1}X^T) \xi.$
Since $(I - X(X^T X)^{-1}X^T)$ is almost surely an idempotent matrix with rank $n-d$, it follows that the distribution of 
\begin{equation*}
    \frac{n\hat L(\wols)}{\sigma^2} = \frac{1}{\sigma^2} \, \xi^T (I - X(X^T X)^{-1}X^T) \xi,
\end{equation*}
is a Chi-square distribution with $n-d$ degrees of freedom. By the same reasoning, the distribution of 
\begin{equation*}
    \frac{n \norm{\wols-w^*}_{\hat \Sigma}^2 }{\sigma^2} = \frac{1}{\sigma^2} \, \xi^T X (X^T X)^{-1} X^T \xi
\end{equation*}
is a Chi-square distribution with $d$ degrees of freedom. By \cref{lem:norm-concentration}, with probability at least $1-\delta$, it holds that
\begin{equation*}
    \frac{\sqrt{n}}{\sigma} \sqrt{\hat L(\wols)}\leq \sqrt{n-d} + 2\sqrt{\log(4/\delta)}.
\end{equation*}
Similarly, we have
\begin{equation*}
    \frac{\sqrt{n}}{\sigma} \norm{\wols-w^*}_{\hat \Sigma} \leq \sqrt{d} + 2\sqrt{\log(4/\delta)} .
\end{equation*}
Rearranging the terms conclude the proof.
\end{proof}

\OLSbound*

\begin{proof}
By \cref{lem:norm-concentration}, we can pick
\begin{equation*}
    \begin{split}
        F(w) &= \left(\sqrt{d} + 2 \sqrt{\log(4/\delta')}\right) \Norm{\Sigma^{1/2} (w^* - w)}_2 \\
        &= \left(\sqrt{d} + 2 \sqrt{\log(4/\delta')}\right) \sqrt{L(w)-\sigma^2}.\\
    \end{split}
\end{equation*}
Let $\delta' = \delta/9$ and replace $\delta$ by $\delta/3$ in \cref{thm:optimistic}, plug in the estimates from \cref{lem:ols-train-err} using confidence level $\delta/9$, then by a union bound with $\gamma = \frac{d}{n} $ and $ \epsilon = \sqrt{\frac{\log(36/\delta)}{n}}$, we have
\begin{equation} \label{eqn:ols-train-err}
    \sqrt{\hat{L}(\wols)} \leq \sigma \sqrt{1-\gamma} + 2 \sigma \epsilon
\end{equation}
and the bound \eqref{eqn:optimistic-1} becomes
\[
L(w) \leq (1+14 \epsilon) \left( \sqrt{\hat L(w)} + (\sqrt{\gamma}+2\epsilon)\sqrt{L(w)-\sigma^2} \right)^2.
\]
We can simplify this by expanding the square
\[
(1+14 \epsilon)^{-1} L(w) \leq \hat{L}(w) + (\sqrt{\gamma}+2\epsilon)^2 (L(w)-\sigma^2) + 2 (\sqrt{\gamma}+2\epsilon) \sqrt{\hat{L}(w)} \sqrt{L(w)-\sigma^2} .
\]
Rearranging, we arrive at
\[
\left[ (1+14 \epsilon)^{-1} - (\sqrt{\gamma}+2\epsilon)^2 \right] (L(w)-\sigma^2) \leq \hat{L}(w) - (1+14 \epsilon)^{-1}\sigma^2 + 2 (\sqrt{\gamma}+2\epsilon) \sqrt{\hat{L}(w)} \sqrt{L(w)-\sigma^2} .
\]
Note that this is a quadratic equation in terms of $\sqrt{L(w)- \sigma^2}$
\[
(L(w)-\sigma^2) - 2 \frac{(\sqrt{\gamma}+2\epsilon)\sqrt{\hat L(w)}}{(1+14 \epsilon)^{-1} - (\sqrt{\gamma}+2\epsilon)^2} \sqrt{L(w)-\sigma^2} \leq \frac{\hat{L}(w) - (1+14 \epsilon)^{-1} \sigma^2}{(1+14 \epsilon)^{-1} - (\sqrt{\gamma}+2\epsilon)^2}.
\]
We can complete the square, which leads to the following
\[
\left[ \sqrt{L(w)-\sigma^2} - \frac{(\sqrt{\gamma}+2\epsilon) \sqrt{\hat L(w)}}{(1+14 \epsilon)^{-1} - (\sqrt{\gamma}+2\epsilon)^2}  \right]^2 \leq  \frac{(1+14 \epsilon)^{-1}}{(1+14 \epsilon)^{-1} - (\sqrt{\gamma}+2\epsilon)^2} \left( \frac{\hat{L}(w)}{(1+14 \epsilon)^{-1} - (\sqrt{\gamma}+2\epsilon)^2} - \sigma^2 \right)
\]
Observe that $(1+14 \epsilon)^{-1} - (\sqrt{\gamma}+2\epsilon)^2 = 1-\gamma - O(\epsilon)$ and so
\[
\frac{\sqrt{\gamma}}{1-\gamma} \leq \frac{(\sqrt{\gamma}+2\epsilon)}{(1+14 \epsilon)^{-1} - (\sqrt{\gamma}+2\epsilon)^2} \leq \frac{\sqrt{\gamma}}{1-\gamma} + O(\epsilon).
\]
We can handle the other terms similarly. Plugging in \eqref{eqn:ols-train-err} concludes the proof.
\end{proof}

\subsection{Minimum-Norm Interpolation with Isotropic Covariance}
\begin{lemma}\label{lem:isotropic-bound} 
Let $w^*,w$ be arbitrary vectors with $w^* \ne 0$, let $V$ be the (one-dimensional) span of $w^*$, and let $P_V$ be the orthogonal projection onto $V$. Then for any vector $x$,
\[\langle w - w^*, x \rangle \leq \| w - w^*\|_2 \cdot \| P_{V} x\|_2 + \|x\|_2\sqrt{\|w\|_2^2 - \frac{\left( \|w\|_2^2 + \|w^*\|_2^2 - \|w - w^*\|_2^2\right)^2}{4 \|w^*\|_2^2}}. \]
\end{lemma}

\begin{proof}
 Observe that by expanding the square, we have
\[ \|w - w^*\|_2^2 
 =  \|w\|_2^2 + \|w^*\|_2^2 - 2 \langle P_V w, w^* \rangle \]
 and so rearranging gives the Parallelogram identity
 \[ \|w\|_2^2 + \|w^*\|_2^2 - \|w - w^*\|_2^2 = 2 \langle P_V w, w^* \rangle. \]
Taking absolute value of both sides and using that $P_V w$ and $w^*$ are colinear gives
\[ \left|\|w\|_2^2 + \|w^*\|_2^2 - \|w - w^*\|_2^2\right| = 2 \|P_V w\|_2 \|w^*\|_2. \]
Combining this with the Pythagorean Theorem, we find
\[ \|P_{V^{\perp}} w\|_2^2 = \|w\|_2^2 - \|P_V w\|_2^2 = \|w\|_2^2 - \left(\frac{\left|\|w\|_2^2 + \|w^*\|_2^2 - \|w - w^*\|_2^2\right|}{2 \|w^*\|}\right)^2. \]
Thus, applying the Cauchy-Schwarz inequality and the above gives
\begin{align*}
\langle w - w^*, x \rangle 
&= \langle P_{V} (w - w^*), x \rangle + \langle P_{V^{\perp}} w, x \rangle  \\
&\le \langle P_{V} (w - w^*), x \rangle + \|P_{V^{\perp}} w\|_2 \|x\|_2  \\
&= \langle w - w^*, P_{V} x \rangle  + \|x\|_2\sqrt{\|w\|_2^2 - \frac{\left( \|w\|_2^2 + \|w^*\|_2^2 - \|w - w^*\|_2^2\right)^2}{4 \|w^*\|_2^2}}\\
&\leq \| w - w^*\|_2 \cdot \| P_{V} x\|_2 + \|x\|_2\sqrt{\|w\|_2^2 - \frac{\left( \|w\|_2^2 + \|w^*\|_2^2 - \|w - w^*\|_2^2\right)^2}{4 \|w^*\|_2^2}}.
\end{align*}
which is the desired inequality.
\end{proof}

\begin{lemma}\label{lem:isotropic-generalization}
Under the assumptions of \cref{thm:optimistic} with $\gamma = d/n > 1$ and the further assumption that the data has isotropic covariance $\Sigma = I_d$, there exists $\epsilon \lesssim \sqrt{\frac{\log(18/\delta)}{n}}$ such that with probability at least $1-\delta$, we have
\[
\|w - w^*\|_2^2 + \sigma^2
\leq (1+\epsilon) \left( \sqrt{\hat{L}(w)} + \sqrt{\gamma} \cdot \sqrt{\|w\|_2^2 - \frac{\left( \|w\|_2^2 + \|w^*\|_2^2 - \|w - w^*\|_2^2\right)^2}{4 \|w^*\|_2^2}} \right)^2.
\]
\end{lemma}

\begin{proof}
Observe that $\frac{\langle w^*, x\rangle}{\| w^*\|_2}  \sim \mathcal{N}(0, 1)$ and so by a standard Gaussian tail bound, \cref{lem:norm-concentration} and a union bound, with probability at least $1-\delta$, it holds that
\[
\|P_V x\|_2 = \norm{\frac{ w^*(w^*)^T}{\| w^*\|_2^2} x}_2 = \frac{|\langle w^*, x\rangle|}{\| w^*\|_2} \leq \sqrt{2\log(6/\delta)}
\]
and 
\[
\| x\|_2 \leq \sqrt{d} + 2 \sqrt{\log(6/\delta)}.
\]

Combining \cref{lem:isotropic-bound} with \cref{thm:optimistic} and another union bound gives

\begin{align*} 
&\frac{1}{\sqrt{1 + \beta_1}} \sqrt{\|w - w^*\|_2^2 + \sigma^2}  \\
&\le \sqrt{\hat{L}(w)} + \|w - w^*\|_2 \sqrt{\frac{2\log(18/\delta)}{n}} + \left(\sqrt{\frac{d}{n}} + 2 \sqrt{\frac{\log(18/\delta)}{n}} \right) \sqrt{\|w\|_2^2 - \frac{\left( \|w\|_2^2 + \|w^*\|_2^2 - \|w - w^*\|_2^2\right)^2}{4 \|w^*\|_2^2}}.
\end{align*}

Using the fact that $\|w - w^*\|_2 \leq \sqrt{\|w - w^*\|_2^2 + \sigma^2}$ and $d > n$, we have

\begin{align*} 
&\left(1 + 2 \sqrt{\frac{\log(18/\delta)}{n}} \right)^{-1} \left( \frac{1}{\sqrt{1 + \beta_1}} - \sqrt{\frac{2\log(18/\delta)}{n}} \right) \sqrt{\|w - w^*\|_2^2 + \sigma^2}  \\
&\le \sqrt{\hat{L}(w)}  + \sqrt{\gamma} \cdot  \sqrt{\|w\|_2^2 - \frac{\left( \|w\|_2^2 + \|w^*\|_2^2 - \|w - w^*\|_2^2\right)^2}{4 \|w^*\|_2^2}}.
\end{align*}

To simplify, there exists $\epsilon \lesssim \sqrt{\frac{\log(18/\delta)}{n}}$ such that
\[
\frac{1}{\sqrt{1 + \epsilon}} \sqrt{\|w - w^*\|_2^2 + \sigma^2} 
\leq \sqrt{\hat{L}(w)} + \sqrt{\gamma} \cdot \sqrt{\|w\|_2^2 - \frac{\left( \|w\|_2^2 + \|w^*\|_2^2 - \|w - w^*\|_2^2\right)^2}{4 \|w^*\|_2^2}}.
\]
and rearranging concludes the proof.
\end{proof}

The generalization bound from \cref{lem:isotropic-generalization} holds for all $w$; we now show what happens when we specialize it to interpolators.

\pagebreak
\RidgeGen*

\begin{proof}
By \cref{lem:isotropic-generalization}, there exists some $\epsilon \lesssim \sqrt{\frac{\log(18/\delta)}{n}}$ such that with probability at least $1-\delta$, for all $w$ such that $\hat{L}(w) = 0$ it holds that
\begin{equation*}
    \begin{split}
        \|w - w^*\|_2^2 + \sigma^2
        &\leq (1+\epsilon)\gamma \left( \|w\|_2^2 - \frac{\left( \|w\|_2^2 + \|w^*\|_2^2 - \|w - w^*\|_2^2\right)^2}{4 \|w^*\|_2^2}\right) \\
        &= (1+\epsilon)\gamma  \left( \|w\|_2^2 - \frac{(\|w\|_2^2 + \|w^*\|_2^2)^2 - 2(\|w\|_2^2 + \|w^*\|_2^2)\|w - w^*\|_2^2 + \|w - w^*\|_2^4}{4 \|w^*\|_2^2}\right) \\
    \end{split}
\end{equation*}
Rearranging, we have
\[
4 \|w^*\|_2^2 \cdot \frac{\|w - w^*\|_2^2 + \sigma^2}{(1+\epsilon)\gamma} 
\leq 4 \|w^*\|_2^2 \cdot \|w\|^2 - (\|w\|_2^2 + \|w^*\|_2^2)^2 + 2(\|w\|_2^2 + \|w^*\|_2^2)\|w - w^*\|_2^2 - \|w - w^*\|_2^4
\]
Grouping the terms with $\| w-w^*\|_2^2$, we see that
\[
\|w - w^*\|_2^4 + \left( \frac{4 \|w^*\|_2^2 }{(1+\epsilon)\gamma}  - 2(\|w\|_2^2 + \|w^*\|_2^2) \right) \cdot \|w - w^*\|_2^2+ 4 \|w^*\|_2^2 \cdot \frac{\sigma^2}{(1+\epsilon)\gamma}
\leq 4 \|w^*\|_2^2 \cdot \|w\|^2 - (\|w\|_2^2 + \|w^*\|_2^2)^2 
\]
which is equivalent to
\[
\|w - w^*\|_2^4 - 2\left( \|w\|_2^2 + \left( 1 -\frac{2 }{(1+\epsilon)\gamma}  \right) \|w^*\|_2^2 \right) \cdot \|w - w^*\|_2^2 + (\|w\|_2^2 - \|w^*\|_2^2)^2 + 4 \|w^*\|_2^2 \cdot \frac{\sigma^2}{(1+\epsilon)\gamma}
\leq 0.
\]
To complete the square, we compute
\begin{equation*}
    \begin{split}
        &\left( \|w\|_2^2 + \left( 1 -\frac{2 }{(1+\epsilon)\gamma}  \right) \|w^*\|_2^2 \right)^2 - (\|w\|_2^2 - \|w^*\|_2^2)^2 - 4 \|w^*\|_2^2 \cdot \frac{\sigma^2}{(1+\epsilon)\gamma}\\
        =& \left( 1 -\frac{2 }{(1+\epsilon)\gamma}  \right)^2 \|w^*\|_2^4 + 2 \left( 1 -\frac{2 }{(1+\epsilon)\gamma}  \right) \|w\|_2^2 \|w^*\|_2^2 - \|w^*\|_2^4 + 2 \|w\|_2^2\|w^*\|_2^2 - 4 \|w^*\|_2^2 \cdot \frac{\sigma^2}{(1+\epsilon)\gamma}\\
        =& \left( \frac{4 }{(1+\epsilon)^2\gamma^2 }- \frac{4}{(1+\epsilon)\gamma}  \right) \|w^*\|_2^4 + 4 \left( 1 -\frac{1 }{(1+\epsilon)\gamma}  \right) \|w\|_2^2 \|w^*\|_2^2 - 4 \|w^*\|_2^2 \cdot \frac{\sigma^2}{(1+\epsilon)\gamma}\\
        =& 4 \|w^*\|_2^2 \left[ \left( 1 -\frac{1 }{(1+\epsilon)\gamma}  \right) \left( \|w\|_2^2 - \frac{\|w^*\|_2^2}{(1+\epsilon)\gamma} \right) -  \frac{\sigma^2}{(1+\epsilon)\gamma} \right]\\
        \leq & 4 \|w^*\|_2^2 \left[ \left( 1 +\epsilon -\frac{1 }{\gamma}  \right) \left( \|w\|_2^2 +\epsilon \|w\|_2^2- \frac{\|w^*\|_2^2}{\gamma} \right) -  \frac{\sigma^2}{\gamma} \right]\\
    \end{split}
\end{equation*}
where in the last step we use $(1+\epsilon)^2 \geq 1$ and $\frac{\sigma^2(1+\epsilon)}{\gamma} \geq \frac{\sigma^2}{\gamma}$. To simplify, it is routine to check that
\[
\left( 1 +\epsilon -\frac{1 }{\gamma}  \right) \left( \|w\|_2^2 +\epsilon \|w\|_2^2- \frac{\|w^*\|_2^2}{\gamma} \right) - \left( 1  -\frac{1 }{\gamma}  \right) \left( \|w\|_2^2 - \frac{\|w^*\|_2^2}{\gamma} \right) \leq 3\epsilon \| w \|_2^2
\]
and so we can conclude that
\[
\left| \|w - w^*\|_2^2 - \left[ \|w\|_2^2 + \left( 1 -\frac{2 }{(1+\epsilon)\gamma}  \right) \|w^*\|_2^2 \right] \right| \leq 2\|w^*\|_2 \sqrt{\left( 1  -\frac{1 }{\gamma}  \right) \left( \|w\|_2^2 - \frac{\|w^*\|_2^2}{\gamma} \right)- \frac{\sigma^2}{\gamma} +  3\epsilon \| w \|_2^2}.
\]
as desired.
\end{proof}

\IsotropicNorm*

\begin{proof}
The proof strategy here follows the same lines as in Theorem 2 of \citet{uc-interpolators}, but handles the $w^*$ term more carefully. First, we introduce the Lagrangian and apply a change of variable
\begin{equation*}
    \begin{split}
        \min_{Xw= Y} \| w\|^2 &= \min_w \max_{\lambda} \, \langle \lambda, Xw - Y \rangle + \| w\|^2 \\
        &= \min_w \max_{\lambda} \, \langle \lambda, Xw - \xi \rangle + \|w + w^* \|^2 \\
    \end{split}
\end{equation*}
To apply CGMT (\cref{thm:gmt}), we need a double truncation argument. For any $r, t > 0$, introduce the following problem:
\begin{equation}
    \Phi_r(t) = \min_{\|w + w^* \|^2 \leq 2t} \max_{\|\lambda\| \leq r} \, \langle \lambda, Xw - \xi \rangle + \| w + w^* \|^2.
\end{equation}
We also introduce 
\begin{equation}
    \begin{split}
        \Phi(t) 
        &= \min_{\|w + w^* \|^2 \leq 2t} \max_{\lambda} \, \langle \lambda, Xw - \xi \rangle + \| w + w^* \|^2 \\
        &= \min_{\substack{Xw = \xi \\ \| w + w^* \|^2 \leq 2t}} \, \| w + w^* \|^2 \\
    \end{split}
\end{equation}
and claim that $\Phi_r(t) \to \Phi(t)$ as $r \to \infty$. By definition, $\Phi_r(t) \leq \Phi_s(t)$ for $r \leq s$. We consider two cases:
\begin{enumerate}
    \item $\Phi(t) = \infty$, i.e. the minimization problem defining $\Phi(t)$ is infeasible. In this case, we know that for all $\|w + w^*\|^2 \le 2t$
    \begin{equation*}
        \| X w - \xi\|_2 > 0.
    \end{equation*}
    By compactness, there exists $\mu = \mu(X, \xi) > 0$ (in particular, independent of $r$) such that
    \begin{equation*}
        \| X w - \xi \|_2 \geq \mu.
    \end{equation*}
    Therefore, considering $\lambda$ along the direction of $ X w - \xi$ shows that
    \begin{equation*}
        \Phi_r(t) = \min_{\| w + w^* \|^2 \leq 2t} \max_{\|\lambda\|_2 \leq r} \, \langle \lambda, Xw - \xi \rangle + \| w + w^* \|^2 \geq r\mu 
    \end{equation*}
    so $\Phi_r(t) \to \infty$ as $r \to \infty$. 
    
    \item Otherwise $\Phi(t) < \infty$, i.e.\ the minimization problem defining $\Phi(t)$ is feasible. In this case, we can let $w(r)$ be an arbitrary minimizer achieving the objective $\Phi_r(t)$ for each $r \geq 0$ by compactness. By compactness again, the sequence $\{ w(r) \}_{r=1}^{\infty}$ at positive integer values of $r$ has a subsequential limit $w(\infty)$ such that $\| w(\infty) + w^* \| \leq 2t$. Equivalently, there exists an increasing sequence $r_n$ such that $\lim_{n \to \infty} w(r_n) = w(\infty)$.
    
    Suppose for the sake of contradiction that $X w(\infty) \ne \xi$, then by continuity, there exists $\mu > 0$ and a sufficiently small $\epsilon > 0$ such that for all $\|w - w(\infty) \|_2 \leq \epsilon$
    \begin{equation*}
        \|Xw - \xi \|_2 \geq \mu
    .\end{equation*}
    This implies that for sufficiently large $n$, we have
    \begin{equation*}
        \|Xw(r_n) - \xi \|_2 \geq \mu
    \end{equation*}
    and by the same argument as in the previous case
    \begin{equation*}
        \Phi_{r_n}(t) = \max_{\|\lambda\|_2 \leq r} \, \langle \lambda, Xw(r_n) - \xi \rangle + \| w(r_n) + w^* \|^2 \geq r\mu
    \end{equation*}
    so $\Phi_{r_n} \to \infty$, but this is impossible since $\Phi_r(t) \leq \Phi(t) < \infty$. By contradiction, it must be the case that $Xw(\infty) = \xi$. By taking $\lambda = 0$ in the definition of $\Phi_r(t)$, we have
    \begin{equation*}
        \Phi_{r_n}(t) \geq \| w(r_n) + w^* \|^2
    .\end{equation*}
    By continuity, we show that 
    \begin{equation*}
        \lim \inf_{n \to \infty} \Phi_{r_n}(t) \geq \lim_{n \to \infty} \|  w(r_n) + w^* \|^2 = \| w(\infty) + w^* \|^2 \geq \Phi(t)
    .\end{equation*}
    Since $\Phi_{r_n}(t) \leq \Phi(t)$, the limit of $\Phi_{r_n}(t)$ exists and equals $\Phi(t)$. We can conclude that $\lim_{r \to \infty} \Phi_r(t) = \Phi(t)$ because $\Phi_r(t)$ is an increasing function of $r$.
\end{enumerate}
In both cases, we have $\Phi_r(t) \to \Phi(t)$ as $r \to \infty$. The auxiliary problem corresponding to $\Phi_r(t)$ is
\begin{equation}
    \phi_r(t) = \min_{\| w + w^* \|^2 \leq 2t} \max_{\|\lambda\|_2 \leq r} \, \| \lambda \| \langle H, w \rangle + \| w\| \langle G, \lambda \rangle -\langle \lambda, \xi \rangle + \|  w + w^* \|^2 
\end{equation}
which is upper bounded by 
\begin{equation}
    \begin{split}
        \phi(t) 
        &= \min_{\| w + w^* \|^2 \leq 2t} \max_{\lambda} \, \| \lambda \| \langle H, w \rangle + \| w\| \langle G, \lambda \rangle -\langle \lambda, \xi \rangle + \| w + w^* \|^2 \\
        &= \min_{\substack{\langle H, w \rangle + \| G \| w\| - \xi\| \leq 0 \\ \| w + w^* \|^2 \leq 2t}} \, \| w + w^* \|^2. \\
    \end{split}
\end{equation}
Applying CGMT and the fact that $\Phi_r(t)$ monotonically increases to $\Phi(t)$ almost surely, we can conclude
\begin{equation*}
    \begin{split}
        \Pr \left( \min_{Xw= Y} \| w\|^2 > t \, | \, \xi \right) &= \Pr \left( \Phi(t) > t \, | \, \xi \right) = \Pr \left( \lim_{r \to \infty} \Phi_r(t) > t \, | \, \xi \right) \\
        &\leq \lim_{r \to \infty} \Pr \left(  \Phi_r(t) > t \, | \, \xi \right) \\
        &\leq 2 \cdot \lim_{r \to \infty} \Pr \left( \phi_r(t) > t \, | \, \xi \right) \\
        &\leq 2 \cdot \Pr \left( \phi(t) > t \, | \, \xi \right) = 2 \cdot \Pr \left( \min_{\langle H, w \rangle + \| G \| w\| - \xi\| \leq 0 } \, \| w + w^* \|^2 > t \, | \, \xi \right) \\
    \end{split}
\end{equation*}
By tower law, we have shown that
\[
\Pr \left( \min_{Xw= Y} \| w\|^2 > t \right) \leq 2 \cdot \Pr \left( \min_{ \| G \| w\| - \xi\| \leq \langle H, w \rangle } \, \| w + w^* \|^2 > t  \right) .
\]
To upper bound the minimum, we consider $w$ of the form $\alpha w^* + \beta PH$ where $P = I - \frac{w^*(w^*)^T}{\|w^* \|^2}$. For the simplicity of notation, define
\[
\epsilon = 2 \sqrt{\frac{\log(40/\delta)}{n}} \quad \text{ and } \quad \rho = \sqrt{\frac{1}{n}} + 2 \sqrt{\frac{\log(20/\delta)}{n}}.
\]
By a union bound, the following collection of events occurs with probability at least $1-\delta/2$:
\begin{enumerate}
    \item By \cref{lem:lowrank-projection}, it holds that
    \[
    | \langle \xi, G \rangle | \leq \rho \| \xi \| \cdot \|G\|
    \]
    \item By \cref{lem:norm-concentration}, it holds that
    \[
    (1-\epsilon) \sigma\sqrt{n} \leq \| \xi \| \leq (1+\epsilon) \sigma\sqrt{n}
    \]
    \[
    (1-\epsilon) \sqrt{n} \leq \| G \| \leq (1+\epsilon) \sqrt{n}
    \]
    \[
    \left(\sqrt{\frac{d-1}{n}}-\epsilon\right) \sqrt{n} 
    \leq \| PH \| 
    \leq \left(\sqrt{\frac{d-1}{n}}+\epsilon\right) \sqrt{n}
    \]
    \item By standard Gaussian tail bound, it holds that
    \[
    |\langle H, w^* \rangle| \leq \| w^* \| \epsilon \sqrt{n}
    \]
\end{enumerate}
The above bounds imply that
\begin{equation*}
    \begin{split}
        \| G \| w\| - \xi\|^2 &= \| G\|^2 \| w\|^2 + \| \xi \|^2 - 2 \| w\| \langle G, \xi \rangle \\
        &\leq (1+\rho) (\| G\|^2 \| w\|^2 + \| \xi \|^2)\\
        &\leq (1+\rho)(1+\epsilon)^2 n (\| w\|^2 + \sigma^2).
    \end{split}
\end{equation*}
By orthogonality, observe that
\[
\| w \|^2 = \alpha^2 \| w^* \|^2 + \beta^2 \| PH\|^2
\]
\[
\langle H, w \rangle = \alpha \langle H, w^* \rangle + \beta \| PH \|^2,
\]
and so to ensure that $\| G \| w\| - \xi\| \leq \langle H, w \rangle$, we can choose $\beta$ such that
\[
(1+\rho)^{1/2}(1+\epsilon) \sqrt{n (\alpha^2 \| w^*\|^2 + \beta^2 \| PH\|^2 + \sigma^2)} + \alpha  \| w^* \| \epsilon \sqrt{n} \leq \beta \| PH\|^2.
\]
Note that it suffices to have
\[
(1+\rho)^{1/2}(1+2\epsilon) \sqrt{n (\alpha^2 \| w^*\|^2 + \beta^2 \| PH\|^2 + \sigma^2)} \leq \beta \| PH\|^2 
\]
\[
\iff
\alpha^2 \frac{ \| w^*\|^2 }{ (1+\rho)^{-1}(1+2\epsilon)^{-2}\frac{\| PH\|^2}{n} - 1 }
+
\frac{ \sigma^2 }{ (1+\rho)^{-1}(1+2\epsilon)^{-2} \frac{\| PH\|^2}{n} - 1 }
\leq \beta^2 \| PH\|^2
\]
Again, by orthogonality, we have
\[
\| w + w^* \|^2 = (1+\alpha)^2 \| w^* \|^2 + \beta^2 \| PH \|^2
\]
and so
\begin{equation*}
    \begin{split}
        &\min_{ \| G \| w\| - \xi\| \leq \langle H, w \rangle } \, \| w + w^* \|^2 \\
        \leq &\frac{ \sigma^2 }{ (1+\rho)^{-1}(1+2\epsilon)^{-2} \frac{\| PH\|^2}{n} - 1 } + \min_{\alpha} \, (1+\alpha)^2 \| w^* \|^2 + \alpha^2 \frac{ \| w^*\|^2 }{ (1+\rho)^{-1}(1+2\epsilon)^{-2}\frac{\| PH\|^2}{n} - 1 }\\
        = &\frac{ \sigma^2 }{ (1+\rho)^{-1}(1+2\epsilon)^{-2} \frac{\| PH\|^2}{n} - 1 } + \frac{ \| w^* \|^2 }{ (1+\rho)^{-1}(1+2\epsilon)^{-2}\frac{\| PH\|^2}{n} }
    \end{split}
\end{equation*}
Finally, we can plug in the high probability lower bound for $\|PH\|\sqrt{n}$ and the proof is complete after some routine calculations.
\end{proof}
\subsection{LASSO with Isotropic Covariance}

\LASSOIsotropic*

\begin{proof}
We use that for $\cK' := \{u: \|w^* + u\|_1 \le \|w^*\|_1 \}$  
\[ \langle w^* - w, x \rangle \le \|w^* - w\| \sup_{u \in \cK' \cap S^{n - 1}} \langle u, x \rangle \]%
where $S^{n - 1}$ is the unit sphere. 
Recall that $\omega := W(\cK' \cap S^{n - 1})$ denotes the Gaussian width of the intersection of the tangent cone $\cK'$ with the unit sphere. 
 Let $\epsilon = \Theta\left(\frac{\log(36/\delta)}{n}\right)^{1/2}$ as in \cref{thm:OLS}, then with this notation
\cref{thm:optimistic} gives
\[ \sigma^2 + \|w^* - w\|_2^2 \le (1 + \beta)\left(\sqrt{\hat{L}(w)} + \|w^* - w\|_2(\omega + 2\epsilon)/\sqrt{n}\right)^2 \le (1 + 14\epsilon)\left(\sqrt{\hat{L}(w)} + \|w^* - w\|_2(\omega + 2\epsilon)/\sqrt{n}\right)^2. \]
This is a quadratic equation in $\|w^* - w\|_2$ which is of exactly the same form as the quadratic equation that arose in the analysis of Ordinary Least Squares (proof of \cref{thm:OLS}), if we define $\gamma = \omega^2/n$. So solving the quadratic equation in the exact same way, we find that under the assumption $\gamma + 2\epsilon/\sqrt{n} < 1$ that
\begin{equation} \label{eqn:lasso-bound}
    \left| \sqrt{L(w)-\sigma^2} - \sqrt{\frac{\gamma\hat{L}(w)}{(1-\gamma)^2}} \, \right| \leq \epsilon \sqrt{\hat{L}(w)} + \sqrt{\frac{1}{1-\gamma} \left( \frac{\hat{L}(w)}{1-\gamma} - \sigma^2 \right) + \epsilon \hat{L}(w)}.
\end{equation}
\end{proof}

\section{Proofs for Section~\ref{sec:local-gw}} \label{apdx:local-gw}
We start with the following result, which lets us upper bound the training error of the ERM in a convex set $\cK$ and is proved using a direct application of the Convex Gaussian Minmax Theorem. 
\ERMPerformance*

\begin{proof}
Observe that
\begin{align}\label{eqn:cgmt-erm}
\min_{w \in \cK} \sqrt{\hat{L}(w)}
&= \frac{1}{\sqrt{n}} \min_{w \in \cK} \max_{\|\lambda\|_2 \le 1} \langle \xi +  Z \Sigma^{1/2} (w^* - w), \lambda \rangle
\end{align}
which is a minimax optimization problem over a convex-conave function on a convex set.
Hence by the Convex Gaussian Minmax Theorem (\cref{thm:gmt}) and the same kind of truncation argument based on \cref{lem:truncation}, to get a probability at least $1 - \delta$ upper bound on the Primary Optimization \eqref{eqn:cgmt-erm}, it suffices to prove a probability at least $1 - \delta/2$ upper bound on the following auxillary problem:
\begin{align*}
	\MoveEqLeft \frac{1}{\sqrt{n}} \min_{w \in \cK} \max_{\|\lambda\|_2 \le 1} \langle \xi, \lambda \rangle + \|\lambda\|_2 \langle H, \Sigma^{1/2}(w^* - w) \rangle + \|w^* - w\|_{\Sigma} \langle G, \lambda \rangle \\
	&= \frac{1}{\sqrt{n}} \min_{w \in \cK} \max_{\|\lambda\|_2 \le 1} \langle \xi + \|w^* - w\|_{\Sigma} G, \lambda \rangle + \|\lambda\|_2 \langle H, \Sigma^{1/2}(w^* - w) \rangle \\
	&= \frac{1}{\sqrt{n}} \min_{w \in \cK} \max\left\{\|\xi + \|w^* - w\|_{\Sigma} G\|_2 + \langle H, \Sigma^{1/2}(w^* - w) \rangle, 0\right\}\\
	&\le \frac{1}{\sqrt{n}} \min_{w \in \cK} \max\left\{(1 + \beta_1) \sqrt{\sigma^2 n + \|w^* - w\|_{\Sigma}^2 n} - \langle H, \Sigma^{1/2}(w - w^*) \rangle, 0 \right\}
\end{align*}
where in the last equality, we used that the maximum is attained along the direction $\xi + \|w^* - w\|_{\Sigma} G$ and is attained either at $\|\lambda\| = 0$ or $\|\lambda\| = 1$. Also, consider two cases: either there exists $w \in \cK$ such that the non-zero quantity inside the max is negative, in which case the minimum is just zero, or for all $w \in \cK$, this quantity is positive and so we can drop the max inside the minimum. In either case, we see that this is not larger than
\[ \max\left\{0, \min_{w \in \cK} (1 + \beta_1) \sqrt{\sigma^2  + \|w^* - w\|_{\Sigma}^2} - \frac{1}{\sqrt{n}} \langle H, \Sigma^{1/2}(w - w^*) \rangle \right\}. \]
For any particular $r \ge 0$, we can control it by restricting to $\cK_r$
\begin{equation*}
    \begin{split}
        &\min_{w \in \cK} (1 + \beta_1) \sqrt{\sigma^2  + \|w^* - w\|_{\Sigma}^2} - \frac{1}{\sqrt{n}} \langle H, \Sigma^{1/2}(w - w^*) \rangle\\
        \leq &\min_{w \in \cK_r} (1 + \beta_1) \sqrt{\sigma^2  + r^2} - \frac{1}{\sqrt{n}} \langle H, \Sigma^{1/2}(w - w^*) \rangle\\\
        = &(1 + \beta_1) \sqrt{\sigma^2  + r^2} - \frac{1}{\sqrt{n}} \sup_{\|w^* - w\|_{\Sigma} \leq r} \langle H, \Sigma^{1/2}(w - w^*) \rangle
    \end{split}
\end{equation*}
and so by Gaussian concentration  (\cref{thm:gaussian-concentration})
\[ \min_{w \in \cK}\sqrt{\hat{L}(w)} \le \max\left\{0, (1 + \beta_1) \sqrt{\sigma^2 + r^2} - W_{\Sigma}(\cK_r)/\sqrt{n} + O(r\sqrt{\log(2/\delta)/n}))\right\}. \]
In particular, we can choose the $r$ that minimizes the right hand side, which concludes the proof.
\end{proof}

\begin{lemma}\label{lem:r-loss-convex}
For any $\sigma \ge 0$, the function $r \mapsto \sqrt{\sigma^2 + r^2}$ is strictly increasing, convex, and $1$-Lipschitz on $\mathbb{R}_{\ge 0}$, and also strictly convex if $\sigma > 0$.
\end{lemma}
\begin{proof}
Let $f(r) := \sqrt{\sigma^2 + r^2}$, then
\[ f'(r) = \frac{r}{\sqrt{\sigma^2 + r^2}} \in (0, 1] \]
and
\begin{align*} 
f''(r) 
= \frac{1}{\sqrt{\sigma^2 + r^2}} - \frac{r^2}{(\sigma^2 + r^2)^{3/2}}
= \frac{\sigma^2}{(\sigma^2 + r^2)^{3/2}}
\end{align*}
which is nonnegative, and positive if $\sigma > 0$.
\end{proof}

\begin{lemma}\label{lem:cKr-concave}
If $\cK$ is a convex set in $\mathbb{R}^d$ and
\[ \cK_r := \cK \cap \{w : \|w - w^*\|_{\Sigma} \le r\} \]
then for any $x \in \mathbb{R}^d$, the function
\[ g(r) := \sup_{w \in \cK_r} \langle x, w - w^* \rangle \]
is increasing and concave. In particular, the function $\omega(r) := W_{\Sigma}(\cK_r)$ is increasing and concave.
\end{lemma}

\begin{proof}
Without loss of generality we may assume the set $\cK$ is closed, since replacing $\cK$ by its closure does not change the value of $g(r)$.
The fact that it is increasing is obvious from the definition. 
Let $r = (1 - \lambda)s + \lambda t$ and let $w_s \in \cK_s, w_t \in \cK_t$. Then $w_r := (1 - \lambda)w_s + \lambda w_t$ lies in $\cK_r$ by convexity of $\cK$, and because $\|w_r - w^*\|_{\Sigma} \le (1 - \lambda) \|w_s - w^*\|_{\Sigma} + \lambda \|w_t - w^*\|_{\Sigma}$ by the triangle inequality. Since
\[ \langle w_r - w^*, x \rangle = (1 - \lambda)\langle w_s - w^*, x \rangle + \lambda \langle w_t - w^*, x \rangle \]
and $w_s,w_t$ were arbitrary vectors in $\cK_s,\cK_t$, taking the maximum over $w_s,w_t$ shows
\[ \max_{w \in \cK_r} \langle w - w^*, x \rangle \ge (1 - \lambda) g(s) + \lambda g(t). \qedhere \]
\end{proof}

We now give the main arguments used in the proof of \cref{thm:local}. 
The following lemma shows how to derive lower bounds on the generalization error of the constrained Empirical Risk Minimizer, by formalizing the informal argument from \cref{sec:local-gw}. To avoid having to perform a union bound over all localization radiuses $r$, we show how to get the conclusion by applying \cref{thm:optimistic} for a few carefully chosen values of sets $\cK_r$; this is equivalent to applying \cref{thm:optimistic} once with a simplified version of the ``optimal complexity functional'' described before. 
\begin{lemma}\label{lem:erm-lowerbound}
Suppose that $\mathcal K$ is a convex set and we are under the model assumptions \eqref{eqn:model} and recall summary functionals $\psi^+_{\delta}, \psi^-_{\delta}$ as defined in \eqref{eqn:psi+} and \eqref{eqn:psi-}.  Let $\delta > 0$ be arbitrary, let $\mu^* := \min_{r \ge 0} \psi_{\delta}^+(r)$, and  suppose that $r_- \ge 0, \mu > \mu^*$ and $\eta > 0$ are such that we have $\eta K \le \delta$ for  $K := \left \lceil \frac{r_-}{\mu - \mu^*} \right \rceil$ 
and for all $r \in [0,r_-]$
\[ \min_{r \in [0,r_-]} \psi^-_{\eta}(r) > \mu. \]
Then with probability at least $1 - 2\delta$, the constrained empirical risk minimizer $\hat w = \arg\min_{w \in \cK} \hat{L}(w)$ satisfies
\[ \|\hat w - w^*\|_{\Sigma} > r_-. \]
\end{lemma}
\begin{proof}
Observe that for any fixed value of $r \le r_-$, it follows from  \cref{thm:optimistic} that with probability at least $1 - \eta$ for all $w \in \cK_r$ where $\eta = \delta + \tau/r$
\begin{align}\label{eqn:r-nu}
\sqrt{\hat L(w)} 
&> (1 - \beta_1)\sqrt{\sigma^2 + \|w - w^*\|^2} - W_{\Sigma}(\cK_r)/\sqrt{n} - Cr\sqrt{\log(2/\eta)/n} \\
&\ge \psi_{\eta}^-(r) - (1 - \beta_1)(r - \|w - w^*\|_{\Sigma}) \\
&\ge \mu - (1 - \beta_1)(r - \|w - w^*\|_{\Sigma}) 
\end{align}
where we used the Lipschitz property from \cref{lem:r-loss-convex}. 
We apply this argument for a grid on $[0,r_-]$ which includes the right end point $r_-$ with spacing $\mu - \mu^* < (\mu - \mu^*)/(1 - \beta_1)$, i.e. with $\lceil \frac{r_-}{\mu - \mu^*} \rceil \le K$ many grid points and apply the union bound, it follows that with probability at least $1 - \eta K \ge 1 - \delta$ that for all $w$ with $\|w - w^*\|_{\Sigma} \le r_-$ that
\[ \sqrt{\hat L(w)} > \mu^*. \]
Recall from \cref{thm:erm-performance} that with probability at least $1 - \delta$ the constrained ERM satisfies $\sqrt{\hat L(\hat w)} \le \mu^*$. Thus, by applying the union bound we show that $\|w' - w^*\|_{\Sigma} > r_-$ with probability at least $1 - 2\delta$.
\end{proof}

\LocalThm*

\begin{proof}
We first show the upper bound $\|w - w^*\|_{\Sigma} \le r_+$ for all $w \in \cK$ with $\sqrt{\hat{L}(w)} \le \mu$. If $r_+ = \infty$, then the upper bound is trivial. Otherwise, we have
\begin{equation}\label{eqn:r_+-continuity}
\psi_{\delta}^-(r_+) = \mu 
\end{equation}
by continuity.
Observe that the conclusion of \cref{thm:optimistic} can be written as 
\begin{equation}\label{eqn:rearranged-optimistic}
(1 + \beta)^{-1/2} \sqrt{L(w)} - \frac{F(w)}{\sqrt{n}} \leq \sqrt{\hat L(w)} 
\end{equation}
so taking $F(w) = W(\cK_{r_+}) + C r_+ \sqrt{\log(2/\delta)/n}$ for $w \in \cK_{r_+}$ and $\infty$ outside, applying \cref{thm:optimistic}, and recalling the definition of $\psi_{\delta}^-$ from \eqref{eqn:psi-} and using \eqref{eqn:r_+-continuity} gives 
\begin{equation}\label{eqn:rplus-proof}
\min_{w \in \cK, \|w - w^*\|_{\Sigma} = r_+} \sqrt{\hat{L}(w)} \ge \psi_{\delta}^-(r_+) = \mu
\end{equation}
Also, by definition if $r \geq r^+$, then $\psi_{\delta}^{+}(r) > \psi_{\delta}^{-}(r) \geq \mu > \mu^*$ and so $r$ cannot be the minimizer of $\psi_{\delta}^{+}$, i.e. we have shown $r^* < r^+$, where %
$r^*$ is the minimizer of $\psi_{\delta}^+$ so
\[ \mu^* = \psi_{\delta}^+(r^*) = \min_{r \ge 0} \psi_{\delta}^+(r). \]
Note that since the minimizer $r^* < r^+$, by applying \cref{thm:erm-performance} we have with probability at least $1 - \delta$ that
\begin{align} \label{eqn:psiplus-proof}
\min_{w \in \cK, \|w - w^*\|_{\Sigma} < r_+} \sqrt{\hat{L}(w)} \le  \psi_{\delta}^+(r^*) = \mu^* < \mu.
\end{align}
This establishes the claim $\|w - w^*\|_{\Sigma} \le r$ by convexity: suppose for contradiction there exists $w \in \cK$ such that $\|w - w^*\|_{\Sigma} > r_+$ and $\sqrt{\hat{L}(w)} \le \mu$.  By \eqref{eqn:psiplus-proof}, there exists $w' \in \cK$ with $\|w' - w^*\|_{\Sigma} < r_+$ and $\sqrt{\hat{L}(w)} < \mu$. Therefore, by convexity we conclude that there exists $w''$ which is a convex combination of $w,w'$ such that $\sqrt{\hat{L}(w'')} < \mu$ and $\|w - w^*\|_{\Sigma} = r_+$, but this contradicts \eqref{eqn:rplus-proof}.

Now we show that $\|w - w^*\|_{\Sigma} \ge r_-$ for all $w \in \cK$ with $\sqrt{\hat L(w)} \le \mu$. If $r_- = -\infty$ then the bound is trivial. Otherwise, by continuity
\[ \psi^-_{\tau/r^*}(r_-) = \mu \]
and by definition for all $r < r^*$ we have $\psi^-_{\tau/r^*}(r_-) \ge \mu$. Also, since $\psi^-_{\tau/r^*}(r^*) \le \psi^+_{\delta}(r^*) = \mu^* < \mu$ from the definition, we know that $r^* > r_-$, hence $\tau/r_- > \tau/r^*$ and so 
\[ \min_{r \in [0,r_-]} \psi^-_{\tau/r_-}(r) \ge \min_{r \in [0,r_-]} \psi^-_{\tau/r^*}(r) = \mu. \]
Therefore, we can apply \cref{lem:erm-lowerbound} to conclude that with probability at least $1 - \delta$, the constrained ERM $\hat w = \arg\min_{w \in \cK} \hat{L}(w)$ satisfies 
\[\|\hat w - w^*\|_{\Sigma} > r_-. \] 
By applying \cref{thm:optimistic} analogously to the $r_+$ case, we know that with probability at least $1 - \tau \ge 1 - \delta$,
\begin{equation}\label{eqn:rminus-proof}
\min_{\|w - w^*\|_{\Sigma} = r_-, w \in \mathcal{K}} \sqrt{\hat L(\hat w)} > \mu 
\end{equation}
and since $\mu^* < \mu$, it follows by a convexity argument that for all $w$ with $\|w - w^*\|_{\Sigma} \le r_-$,
\begin{equation}\label{eqn:for-contradiction-mu}
\sqrt{\hat L(w)} > \mu
\end{equation}
which establishes the desired conclusion as the contrapositive. 
The convexity argument is symmetrical to the $r_+$ case: if \eqref{eqn:for-contradiction-mu} is false for some $w$, then interpolating between $w$ and $\hat w$ and observes that there exists a convex combination $w''$ such that $\sqrt{\hat L(w)} \le \mu$ and $\|w'' - w^*\|_{\Sigma} = r_-$, which contradicts \eqref{eqn:rminus-proof}.%
\end{proof}

\pagebreak
\section{Proofs for Section~\ref{sec:improved-rate}} \label{sec:proof-improved-rate} 
\subsection{Faster Rates for Low-Complexity Classes}

\begin{restatable}{lemma}{TwoBound} \label{thm:2bound}
Under the assumptions of \cref{thm:optimistic} and with the definition of $\beta_1$ there, with probability at least $1 - 4(\delta + \delta')$
\[ 
L(\hat w) \le \sigma^2 + (1 + 2\beta_1) \left(\sqrt{\sigma F(\hat w)/\sqrt{n}} + F(\hat w)/\sqrt{n}\right)^2 
\]
where $\hat w$ is any empirical risk minimizer over a closed convex set $\cK$ containing $w^*$, i.e. $\hat{L}(\hat w) = \min_{w \in \cK} \hat{L}(w)$.
\end{restatable}

\begin{proof}
Write $X = Z \Sigma^{1/2}$ with $Z$ a matrix of i.i.d. Gaussians, and observe
\[ \frac{1}{n} \langle Z^T\xi, \Sigma^{1/2}(w - w^*) \rangle  = \frac{1}{n} \langle \xi, Z \Sigma^{1/2}(w - w^*) \rangle = \frac{1}{n} \langle \xi, X(w - w^*) \rangle \]
Note that conditional on $\xi$, $Z^T \xi$ is just a standard Gaussian $N(0,\|\xi\|_2^2 I_d)$. So with probability at least $1 - \delta'$ (recalling the defining property of the complexity functional $F$) we have
\begin{equation}\label{eqn:conditional-noise}
\frac{1}{n} \langle Z^T\xi, \Sigma^{1/2}(w - w^*) \rangle \le \frac{\|\xi\|_2}{n} F(w). 
\end{equation}
Observe that
\[ \nabla_w \hat{L}(w) = \frac{1}{n} \nabla_w \|Y - Xw\|_2^2 = -\frac{2}{n} X^T(Y - X w) = -\frac{2}{n}(X^T \xi + X^TX (w^* - w)) \]
so from the KKT condition $\langle w^* - \hat w, \nabla_w \hat{L}(\hat w) \rangle \ge 0$ we have
\[ \langle w^* - \hat{w}, X^T \xi \rangle + \langle w^* -w, X^TX (w^* - w) \rangle \le 0\]
so rearranging gives the first inequality, and using \eqref{eqn:conditional-noise} gives the second inequality in
\[ \|w^* - \hat w\|_{\hat \Sigma} \le  \sqrt{\frac{1}{n} \langle \xi, X(\hat w - w^*) \rangle} \le \sqrt{\frac{\|\xi\|_2}{n} F(w)}.\]
By \cref{thm:optimistic} (defining $F(w) = \infty$ outside of $\cK$), for all $w \in \cK$
\[ \|w^* - w\|_{\Sigma} \le (1 + \beta_1)\left[\|w^* - w\|_{\hat \Sigma} + F(w)/\sqrt{n}\right] \]
and so for $\hat w$ we have
\begin{align*} 
\|w^* - \hat w\|_{\Sigma} 
&\le (1 + \beta_1)\left[\|w^* - \hat w\|_{\hat \Sigma} + F(\hat w)/\sqrt{n}\right] \\
&\le (1 + \beta_1)\left[\sqrt{\frac{\|\xi\|_2}{n} F(\hat w)} + F(\hat w)/\sqrt{n}\right] 
\end{align*}
and using the fact that the norm $\|\xi\|_2$ concentrates about $\sigma \sqrt{n}$ by \cref{lem:norm-concentration} and recalling the definition of $\beta_1$, we have
\[ \|w^* - \hat w\|_{\Sigma}^2 \le (1 + 2\beta_1) \left(\sqrt{\sigma F(\hat w)/\sqrt{n}} + F(\hat w)/\sqrt{n}\right)^2. \]
Finally, recalling that $L(\hat w) = \sigma^2 + \|w - \hat w\|_{\Sigma}^2$ gives the bound as claimed. 
\end{proof}

\LowComplexity*

\begin{proof}
Defining $\rho := \sqrt{p/n}$ and \cref{thm:2bound} gives
\[ \|w^* - \hat{w}\|_{\Sigma} \le (1 + 2\beta_1)^{1/2} \left(\sqrt{\sigma F(\hat w)/\sqrt{n}} + F(\hat w)/\sqrt{n}\right) = (1 + 2\beta_1)^{1/2} \left(\sqrt{\sigma \rho \|w - w^*\|_{\Sigma}} + \rho \|w - w^*\|_{\Sigma} \right)\]
hence
\[ (1 - (1 + 2\beta_1)^{1/2} \rho) \|w - w^*\|_{\Sigma} \le (1 + 2\beta_1)^{1/2} \sqrt{\sigma \rho \|w - w^*\|_{\Sigma}} \]
which is equivalent to 
\[ \|w - w^*\|_{\Sigma} \le \frac{(1 + 2\beta_1) \sigma \rho}{(1 - (1 + 2\beta_1)^{1/2} \rho)^2} \]
and this in turn is equivalent to the final result. 
\end{proof}

\OLSFastRate*

\begin{proof}
Recall from the proof of \cref{thm:OLS} that with probability at least $1 - \delta'$ we have
\[ \langle w - w^*, x \rangle \le \left(\sqrt{d} + 2 \sqrt{\log(4/\delta')}\right) \Norm{\Sigma^{1/2} (w^* - w)}_2 \]
where $\delta' = \delta/9$ so the result follows from \cref{thm:low-complexity} with $\cK = \R^d$. 
\end{proof}

\LASSOFastRate*

\begin{proof}
Recall from the proof of \cref{thm:lasso-compatibility}, more specially \eqref{eqn:F-lasso}, that with probability at least $1 - \delta/8$
\[
\langle w - w^*, x \rangle \le \|w - w^*\|_1 \|x\|_{\infty} \le 2 \|(w - w^*)_S\|_1 \|x\|_{\infty} \le \frac{2k^{1/2}}{\phi(\Sigma,S)} \|w - w^*\|_{\Sigma} \max_i \sqrt{2\Sigma_{ii} \log(16d/\delta)} .
\]
so the result follows from \cref{thm:low-complexity}.
\end{proof}

\subsection{Precise Rates for OLS}

\OLSlowerbd*

\begin{proof}
Consider the following estimator:
\begin{equation*}
    \begin{split}
        w_{\alpha} &= w^* + \alpha \left( \wols - w^* \right) \\
        &= w^* + \alpha (X^TX)^{-1} X^T \xi \\
    \end{split}
\end{equation*}
Then the training error is 
\begin{equation*}
    \begin{split}
        \hat{L}(w_{\alpha}) &= \frac{1}{n} \| Y - X w_{\alpha} \|^2 \\
        &= \frac{1}{n} \| \xi - \alpha X (X^TX)^{-1} X^T \xi  \|^2 \\
        &= \frac{1}{n} \| \left(I-X (X^TX)^{-1} X^T \right) \xi + (1-\alpha) X (X^TX)^{-1} X^T \xi \|^2 \\
        &= \frac{1}{n} \| \left(I-X (X^TX)^{-1} X^T \right) \xi \|^2 + (1-\alpha)^2 \frac{1}{n} \| X (X^TX)^{-1} X^T \xi \|^2 \\
        &= \hat{L}(\wols) + (1-\alpha)^2 \norm{\wols-w^*}_{\hat \Sigma}^2 \\
    \end{split}
\end{equation*}

By \cref{lem:ols-train-err}, with probability at least $1-\delta$, it holds that
\[
\norm{\wols-w^*}_{\hat \Sigma}^2 \leq \sigma^2 \left( \sqrt{\gamma} + 2\sqrt{\frac{\log(4/\delta)}{n}}\right)^2
\]
which can again be upper bounded by, for example, $4\sigma^2\gamma$ for a sufficiently large n. Therefore, we can let 
\[
(1-\alpha)^2 4\sigma^2\gamma = c \cdot \frac{\sigma^2}{\sqrt{n}}
\]
and it suffices to pick
\[
\alpha  = 1 + \sqrt{ \frac{c}{4\gamma}} \cdot \frac{1}{n^{1/4}}.
\]

So if we define $c^{\prime} = 2\sqrt{ \frac{c}{4\gamma}}$, then the excess error of $w_{\alpha}$ satisfies
\begin{equation*}
    \begin{split}
        L(w_{\alpha}) - \sigma^2 &= \| \Sigma^{1/2} (w_{\alpha}-w^*) \|^2 \\
        &= \alpha^2 \| \Sigma^{1/2} (\wols-w^*) \|^2 \\
        &\geq \left( 1 + \frac{c^{\prime}}{n^{1/4}} \right) \cdot L(\wols). 
    \end{split}
\end{equation*}
The last inequality follows from the fact that $L(\wols) \geq \sigma^2$.
\end{proof}

\OLSvar*
\begin{proof}
Write $X = Z\Sigma^{1/2}$ and recall that 
\begin{equation*}
    \begin{split}
        L(\wols) - \sigma^2 &= \norm{\wols-w^*}_{\Sigma}^2 = \Norm{\Sigma^{1/2} (X^TX)^{-1}X^T\xi}_{2}^2 \\
        &= \xi^T Z (Z^TZ)^{-2} Z^T \xi. \\
    \end{split}
\end{equation*}
First, we compute the expectation. By the tower law, we have
\begin{equation*}
    \begin{split}
        \E L(\wols) - \sigma^2 &= \E \left[ \E \left[ \xi^T Z (Z^TZ)^{-2} Z^T \xi \, | \, Z \right] \right] \\
        &= \sigma^2 \E \Tr((Z^TZ)^{-1}) \\
        &= \sigma^2 \Tr(\E \left[ (Z^TZ)^{-1} \right] )
    \end{split}
\end{equation*}
Proposition 2.1 of \citet{vonRosen} shows that
\begin{equation*}
    \E [(Z^TZ)^{-1}] = \frac{1}{n-d-1} I_d,
\end{equation*}
and so
\begin{equation*}
    \E L(\wols) = \sigma^2 + \sigma^2 \frac{d}{n-d-1} = \sigma^2 \frac{n-1}{n-d-1}.
\end{equation*}
To compute the variance, by the law of total variance, we have
\begin{equation*}
    \begin{split}
        \Var(L(\wols)) &= \Var(L(\wols)-\sigma^2) \\
        &= \E \Var(\xi^T Z (Z^TZ)^{-2} Z^T \xi \, | \, Z) + \Var(\E(\xi^T Z (Z^TZ)^{-2} Z^T \xi \, | \, Z))
    \end{split}
\end{equation*}
By the variance formula of Gaussian quadratic form, we have
\begin{equation*}
    \Var(\xi^T Z (Z^TZ)^{-2} Z^T \xi \, | \, Z) = 2\sigma^4 \Tr ( (Z^TZ)^{-2} )
\end{equation*}
Proposition 2.1 of \citet{vonRosen} shows that
\begin{equation*}
    \E [(Z^TZ)^{-2}] = \frac{n-1}{(n-d)(n-d-1)(n-d-3)} I_d,
\end{equation*}
and so
\begin{equation*}
    \E \Var(\xi^T Z (Z^TZ)^{-2} Z^T \xi \, | \, Z) = \frac{2\sigma^4 d(n-1)}{(n-d)(n-d-1)(n-d-3)}.
\end{equation*}
To compute the second term, observe that
\begin{equation*}
    \begin{split}
        \Var(\E(\xi^T Z (Z^TZ)^{-2} Z^T \xi \, | \, Z)) &= \sigma^4 \Var (\Tr((Z^TZ)^{-1})) \\
        &= \sigma^4 \Var (\text{vec}(I_d)^T \text{vec}((Z^TZ)^{-1})) \\
        &= \sigma^4 \text{vec}(I_d)^T \Var (\text{vec}((Z^TZ)^{-1})) \text{vec}(I_d)
    \end{split}
\end{equation*}
Proposition 2.1 of \citet{vonRosen} shows that
\begin{equation*}
    \Var (\text{vec}((Z^TZ)^{-1})) = \frac{I_{d^2} + \sum_{i,j} (e_i \otimes e_j)(e_j^T \otimes e_i^T)}{(n-d)(n-d-1)(n-d-3)} + 2 \frac{\text{vec}(I_d)\text{vec}(I_d)^T}{(n-d)(n-d-1)^2(n-d-3)}
\end{equation*}
and so
\begin{equation*}
    \begin{split}
        \frac{1}{\sigma^4}\Var(\E(\xi^T Z (Z^TZ)^{-2} Z^T \xi \, | \, Z)) 
        &= \frac{2d}{(n-d)(n-d-1)(n-d-3)} + \frac{2d^2}{(n-d)(n-d-1)^2(n-d-3)}\\
        &=  \frac{2d(n-1)}{(n-d)(n-d-1)^2(n-d-3)}.\\
    \end{split}
\end{equation*}
Finally, we have shown that
\begin{equation*}
    \Var(L(\wols)) = 2\sigma^4 \frac{d(n-1)}{(n-d-1)^2(n-d-3)}. \qedhere
\end{equation*}
\end{proof}

\OLShighprob*

\begin{proof}

We are interested in the excess risk:
\[
L(\wols) - \sigma^2 = \| \Sigma^{1/2}(\wols-w^*)\|^2 = \| (Z^TZ)^{-1}Z^T \xi\|^2.
\]
Notice that
\[
\| (Z^TZ)^{-1}Z^T \xi\|^2 = \left( (Z^TZ)^{-1/2}Z^T \xi \right)^T (Z^TZ)^{-1} \left( (Z^TZ)^{-1/2}Z^T \xi \right)
\]
and we have the following equality:
\begin{equation*}
    \begin{split}
        b^T(Z^TZ)^{-1} b 
        &= \max_u -\| Zu\|^2 + 2 \langle u, b \rangle \\
        &= \max_u \min_v \, \| v\|^2 + 2 \langle v, Zu \rangle + 2 \langle u, b \rangle. 
    \end{split}
\end{equation*}
We can plug in $(Z^TZ)^{-1/2}Z^T \xi$ into $b$. The $b$ term may seem a bit complicated, but the key observation is that conditioned on $Z$, the distribution of $(Z^TZ)^{-1/2}Z^T \xi \sim \mathcal{N}(0, \sigma^2 I_d)$ actually does not depend on $Z$, and so they are independent. Therefore, we can condition on $b = (Z^TZ)^{-1/2}Z^T \xi$ and the law of $Z$ remains unchanged. To apply \cref{thm:gmt}, we need use a truncation argument. Define the truncated problem as
\begin{equation}
    \Phi_r = \max_{\|u\| \leq r} \min_{v} \, \| v\|^2 + 2 \langle v, Zu \rangle + 2 \langle u, b \rangle,
\end{equation}
then by \cref{lem:truncation}, we have
\begin{equation*}
    \begin{split}
        &\Pr \left( L(\wols) - \sigma^2 > t  \, | \, (Z^TZ)^{-1/2}Z^T \xi = b \right) \\
        = &\Pr \left(\lim_{r \to \infty} \Phi_r > t \right) \leq \lim_{r \to \infty} \Pr \left( \Phi_r > t \right).\\
    \end{split}
\end{equation*}
Given $u$, the minimizer $v = -Zu$ satisfies $\| v \| \leq r\| Z\|$ and so for any $M > 0$, we have
\begin{equation*}
    \begin{split}
        \Pr \left( \Phi_r > t \right) &\leq \Pr \left( \max_{\|u\| \leq r} \min_{\| v\| \leq rM} \, \| v\|^2 + 2 \langle v, Zu \rangle + 2 \langle u, b \rangle > t \right) + \Pr(\| Z\| \geq M )  \\
        &\leq 2\Pr \left( \max_{\|u\| \leq r} \min_{\| v\| \leq rM} \, \| v\|^2 + 2 \| v \| \langle H, u \rangle + 2 \| u\| \langle G, v \rangle + 2 \langle u, b \rangle > t \right) + \Pr(\| Z\| \geq M ) \\
        &= 2\Pr \left( \max_{\|u\| \leq r} \min_{\| v\| \leq rM} \, \| v\|^2 + 2 \| v \| \left( \langle H, u \rangle -  \| G\| \| u\| \right)  + 2 \langle u, b \rangle > t \right) + \Pr(\| Z\| \geq M ) \\
    \end{split}
\end{equation*}
by Gaussian minimax theorem. On the event that $\| G\| \geq \| H\|$, the minimizer is 
\[
\| v \|= \| G\|\| u \| - \langle H, u \rangle \geq (\|G\|-\| H\|) \|u\| > 0.
\]
At the same time, we have 
\[
\| v \| \leq r(\|G\|+\| H\|)
\]
and so
\begin{equation*}
    \begin{split}
        \Pr \left( \Phi_r > t \right) \leq & \, 2\Pr \left( \max_{\|u\| \leq r} \, 2 \langle u, b \rangle  - \left( \langle H, u \rangle -  \| G\| \| u\| \right)^2  > t , \| G\| > \| H \|\right) + 2 \Pr(\| G\| \leq \| H \| ) \\
        &\qquad + 2 \Pr(\|G\|+\| H\| \geq M) + \Pr(\| Z\| \geq M ). \\
    \end{split}
\end{equation*}
As the max over $\{ u: \|u\| \leq r\}$ is always smaller than the overall max, taking $M \to \infty$, we have 
\[
\Pr \left( \Phi_r > t \right) \leq 2\Pr \left( \max_{u} \, 2 \langle u, b \rangle  - \left( \| G\| \| u\| - \langle H, u \rangle  \right)^2  > t , \| G\| > \| H \| \right) + 2 \Pr(\| G\| \leq \| H \| ) 
\]

Observe that any $u$ can be decomposed into two parts: one part spanned by $b$ and the other part in the orthogonal complement of $b$. Formally, we write $u = \alpha b + k$ where $\langle k, b \rangle = 0$, and the problem becomes
\[
\max_{\alpha \in \R, \langle k, b\rangle=0} \, 2 \alpha \| b\|^2  - \left( \| G \|\cdot \sqrt{\alpha^2 \| b\|^2 + \| k\|^2}  - \langle H, k \rangle - \alpha \langle H, b \rangle \right)^2.
\]
Define $P = I_d - \frac{bb^T}{\|b\|^2}$. On the event that $\|G\| > \| H \|$, the quantity inside the square is always positive and so we want to choose the direction of $k$ that make $\langle H, k \rangle$ as large as possible:
\begin{equation*}
    \begin{split}
        &\max_{\alpha \in \R} \, 2 \alpha \| b\|^2  - \min_{\langle k, b\rangle=0} \left( \| G \|\cdot \sqrt{\alpha^2 \| b\|^2 + \| k\|^2}  - \langle H, k \rangle - \alpha \langle H, b \rangle \right)^2\\
        =&\max_{\alpha \in \R} \, 2 \alpha \| b\|^2  - \left( \min_{\langle k, b\rangle=0}  \| G \|\cdot \sqrt{\alpha^2 \| b\|^2 + \| k\|^2}  - \langle H, k \rangle  - \alpha \langle H, b \rangle \right)^2 \\
        = &\max_{\alpha \in \R} \, 2 \alpha \| b\|^2  - \left( \min_{\beta \geq 0} \, \| G \|\cdot \sqrt{\alpha^2 \| b\|^2 + \beta^2}  - \beta \| PH\| - \alpha \langle H, b \rangle \right)^2 \\
        = &\max_{\alpha \in \R} \, 2 \alpha \| b\|^2  - \left( |\alpha| \cdot \| b\| \sqrt{\|G\|^2 - \|PH\|^2}  - \alpha \langle H, b \rangle \right)^2 \\
        \leq &\max_{\alpha \in \R} \, 2 \alpha \| b\|^2  - \alpha^2 \| b\|^2 \left(  \sqrt{\|G\|^2 - \|PH\|^2}  - \frac{|\langle H, b \rangle|}{\|b\|} \right)^2 = \frac{\|b\|^2}{\left(  \sqrt{\|G\|^2 - \|PH\|^2}  - \frac{|\langle H, b \rangle|}{\|b\|} \right)^2} 
    \end{split}
\end{equation*}

By the tower law, we have shown that
\begin{equation*}
    \begin{split}
        \Pr \left( L(\wols) - \sigma^2 > \frac{\|b\|^2}{t}   \right) 
        &\leq
        2 \pr \left( \| G\| \leq \| H\| \right) + 2 \Pr \left(  \sqrt{\|G\|^2 - \|PH\|^2}  - \frac{|\langle H, b \rangle|}{\|b\|} < \sqrt{t} , \| G\| > \| H \| \right) \\
        &= 2 \Pr \left( \| G\| \leq \| H\| \quad \text{or} \quad  \sqrt{\|G\|^2 - \|PH\|^2}  - \frac{|\langle H, b \rangle|}{\|b\|} < \sqrt{t} , \| G\| > \| H \| \right)
    \end{split}
\end{equation*}

For the simplicity of notation, denote
\[
\epsilon = 2 \sqrt{\frac{\log(32/\delta)}{n}}.
\]
By a union bound, with probability at least $1-\delta/2$, the following occurs:
\begin{enumerate}
    \item by \cref{lem:norm-concentration} and the fact that $b \sim \mathcal{N}(0, \sigma^2 I_d)$, it holds that
    \[
    \| G\|^2 \geq n(1-\epsilon)^2
    \]
    \[
    \| PH\|^2 \leq n(\sqrt{\gamma}+\epsilon)^2 \quad \text{and} \quad \| b\|^2 \leq \sigma^2 n (\sqrt{\gamma}+\epsilon)^2
    \]
    \item As $\frac{\langle H, b \rangle}{\| b\|} \sim \mathcal{N}(0,1)$, by standard Gaussian concentration, it holds that
    \[
    \frac{|\langle H, b \rangle|}{\|b\|} \leq \epsilon \sqrt{n}
    \]
\end{enumerate}
Therefore, for sufficiently large $n$, we have $\| G\| > \| H\|$ and we can pick $t$ by setting
\[
\sqrt{t} = \sqrt{n(1-\epsilon)^2-n (\sqrt{\gamma}+\epsilon)^2} - \epsilon \sqrt{n}
\]
and so with probability at least $1-\delta$, we have 
\[
L(\wols) - \sigma^2 \leq \frac{\sigma^2 (\sqrt{\gamma}+\epsilon)^2}{\left( \sqrt{(1-\epsilon)^2- (\sqrt{\gamma}+\epsilon)^2} - \epsilon \right)^2}.
\]
It is then routine to check the desired bound.
\end{proof}

\end{document}